\documentclass[journal]{IEEEtran}
\usepackage{cite}
\usepackage{amsmath,amssymb,amsfonts}
\usepackage{amsthm}
\usepackage{graphicx}
\usepackage{booktabs}
\usepackage{textcomp}
\usepackage{stfloats}
\usepackage{tikz}
\usetikzlibrary{arrows.meta,positioning,backgrounds,fit,calc}
\usepackage{xcolor}
\usepackage{url}
\usepackage[hidelinks]{hyperref}

\definecolor{rustc}{HTML}{B5531F}
\definecolor{tealc}{HTML}{1F6F74}
\definecolor{badc}{HTML}{C0392B}
\definecolor{goodc}{HTML}{3F7D3A}
\definecolor{inkc}{HTML}{241F1A}

\newtheorem{proposition}{Proposition}
\newtheorem{definition}{Definition}
\newtheorem{assumption}{Assumption}

\newcommand{\za}{z_{\mathrm{a}}}
\newcommand{\zc}{z_{\mathrm{c}}}

\begin{document}

\title{Diagnosing Under-Development of Irreversible Processes in Video Generation}

\author{Jian~Xu,
        Yanning~Wu,
        Delu~Zeng,
        John~Paisley,
        and~Qibin~Zhao
\thanks{Jian Xu is with RIKEN iTHEMS and RIKEN AIP, Japan (e-mail: jian.xu@riken.jp).}%
\thanks{Yanning Wu and Delu Zeng are with the South China University of Technology, China.}%
\thanks{John Paisley is with Columbia University, New York, NY, USA.}%
\thanks{Qibin Zhao is with RIKEN AIP, Japan.}}

\markboth{IEEE Transactions on Multimedia}%
{Xu \MakeLowercase{\textit{et al.}}: Diagnosing Under-Development of Irreversible Processes in Video Generation}

\maketitle

\begin{abstract}
Many physical attributes are \emph{irreversible}: ice melts but does not re-freeze, paper chars but
does not un-burn. Do video generators respect this? We show the question is hard to measure, and that
what can be measured reliably is \emph{development} rather than reversal. Metrics of local reversal are
null-degenerate: a per-clip violation rate scores $0.50$ on pure noise, and a variance-normalized
reversal residual sits at its noise ceiling. What survives null-testing is a two-part protocol: progress
(a directional attribute correlation) and a stasis rate. Under this protocol, generated video separates
cleanly from real footage, and the gap is human-validated. Across seven text-to-video models, real
reference footage advances ($\rho{=}{+}0.40$, $35\%$ static) while every generator shows near-zero
progress and $92$--$100\%$ stasis; nine annotators rate real footage far above generated ($2.75$ vs.\
$0.99$ on a $0$--$4$ scale). The reliable finding is \emph{under-development}: generators barely advance
irreversible attributes rather than reversing them. As a complementary mechanism, we show that post-hoc
readout guidance is gameable, whereas enforcing monotonicity by construction in a disentangled attribute
latent removes the gameable readout, validated in controlled and semi-synthetic settings.
\end{abstract}


\section{Introduction}

A generative model of video is implicitly a model of how the world changes over time. Some of those
changes are reversible (an object moves left, then right; a light brightens, then dims), but many
are \emph{irreversible}: burning, melting, rusting, wilting, aging, decomposition, and the
accumulation of damage all proceed in one direction. A faithful video generator should render such attributes advancing in one direction---a charred sheet of
paper should not become clean, a puddle should not re-form into an ice cube---yet current text-to-video
(T2V) systems largely fail to render these directional changes: benchmarks of object state change and
metamorphic time-lapse generation report that models struggle to produce temporally consistent,
directional state transitions \cite{oscbench,chronomagic}. Crucially, neither the flow-matching /
diffusion training objective nor popular quality metrics explicitly reward development of an irreversible
attribute, so there is no pressure---at training or evaluation time---for a generated process to actually
progress. This paper asks how to \emph{measure} that, and what can be claimed once measured.

A natural first attempt is to \emph{add} that pressure at inference: define a differentiable readout
$a_t$ of the attribute (e.g.\ a vision-language similarity to ``rusted metal''), and steer sampling
so that $a_t$ is non-decreasing, in the spirit of classifier / energy guidance. On the constructive
side, our first result about this fix is a \textbf{negative} one: it is \emph{adversarially gamed}. Optimizing the
generation to satisfy a frozen readout finds imperceptible or off-manifold changes that raise the
readout without producing the actual visual attribute. We demonstrate this with a controlled
injected reversal: guidance drives the readout to its monotone target, but a
\emph{readout-independent} physical probe of the attribute does not move, and the frame stays
visibly in the wrong state (Fig.~\ref{fig:gaming}). We establish a possibility result with a demonstration: any objective that depends on the input only
through the readout is blind to false-high solutions, so guidance inherits exactly the readout's
faithfulness on the generator's reachable set---which for a learned readout is uncertified off the
data manifold, where latent optimization lands. Robust, ensembled, on-manifold or physically-grounded
readouts may narrow the gap; our point is that faithfulness must be \emph{established}, and that this
failure silently corrupts evaluation whenever the same readout also scores success.

The alternative this negative result motivates is to make irreversibility a property of the
\emph{generative parameterization} rather than of an external score. We
posit a disentangled latent $z=(\za,\zc)$ in which $\za$ is the scalar irreversible attribute and
$\zc$ captures nuisance (identity, pose, background, lighting), and we constrain $\za$ to evolve
monotonically \emph{by construction}. Since the decoder is trained to render $\za$ into the true
visual attribute, enforcing $\za$ monotone yields a genuinely monotone \emph{decoded} sequence, and
there is no external readout to game. In a controlled domain with a known ground-truth attribute and
an independent probe, this repairs injected reversals into real attribute change, preserves
identity, and does not degrade reconstruction (Fig.~\ref{fig:monotone})---precisely the regime in
which guidance fails.

We make three contributions. \textbf{(i)} A \emph{null-robust protocol} for irreversibility: null
baselines show the naive metrics are degenerate---a per-clip violation rate and a variance-normalized
reversal residual both score at their noise ceiling on pure noise, and a generic monotonicity score
rewards stasis (a $1{,}050$-clip, six-model demonstration)---so we build the protocol from the two
quantities that survive the null, progress ($\rho$ and magnitude) and a stasis rate. \textbf{(ii)} A
\emph{human-validated seven-model diagnosis}: across CogVideoX-2b and six ChronoMagic-Bench models
(${\approx}560$ generated clips, benchmarked against $108$ real reference clips run through the identical
pipeline and nine human annotators), generations separate cleanly from real footage by
\emph{under-development}---near-zero progress and near-total stasis. \textbf{(iii)} As a \emph{complementary
mechanism}, we show that post-hoc frozen-readout optimization is provably gameable (confirmed across five
readouts and a four-annotator study), which motivates enforcing monotonicity \emph{by construction} in a
disentangled attribute latent; this removes the gameable readout, and we validate it in controlled and
semi-synthetic settings where the required faithful latent is available.

\section{Related Work}

\paragraph{Irreversibility: measured, not enforced.} Benchmarks quantify how well generators render directional state change---object-state transitions such as frying, melting, and slicing \cite{oscbench}, and metamorphic time-lapse of melting, growth, and decay \cite{chronomagic}---while a fast-growing line evaluates physical commonsense and world-model behavior of generated video \cite{motamed2026physicsiq,bansal2025videophy,li2025worldmodelbench}, datasets of irreversible actions supply supervision \cite{changeit}, and arrow-of-time analyses discriminate forward from reversed video \cite{arrowoftime}. These measure or classify directionality; none add a constraint that prevents reversal during synthesis, which is our focus.

\paragraph{Monotone and constrained generation.} Monotone normalizing flows and invertible monotone operators give exact monotone maps \cite{autm,invmonotone}, and projected/constrained diffusion enforces hard constraints during sampling \cite{projgdm}. This machinery is attribute-agnostic; we instantiate it on a learned, disentangled irreversible-attribute latent for video.

\paragraph{Guidance and adjacent lines.} Optimizing a diffusion sample against a learned score can yield off-manifold satisfactions rather than genuine change; our gaming result (Sec.~\ref{sec:gaming}) is this failure specific to a frozen attribute readout. Our combination (video, an enforced monotone irreversible attribute, a disentangled control latent, and the gaming negative) builds on adjacent lines we do not subsume: physics-aware and object-state video generation \cite{physiovg2024,soucek2024genhowto,wang2025wisa,foo2026psivg,shen2026phantom}, temporal-consistency and controllable generation \cite{he2022lvdm,wang2023videocomposer,zhao2026lstd,cai2026fluencyve}, monotonic networks and structured dynamics \cite{sill1997monotonic,runje2023constrained,ayed2019learning}, disentanglement identifiability, concept bottlenecks, and disentangled video customization \cite{locatello2019challenging,khemakhem2020ivae,koh2020concept,chen2025videodreamer}, reward hacking and physics post-training \cite{skalse2022reward,gao2023scaling,li2025pisa}, and temporal alignment \cite{dwibedi2019tcc,hadji2021drift}.

\section{Problem setup}
\label{sec:setup}

\begin{figure*}[t]
\centering
\includegraphics[width=0.49\linewidth]{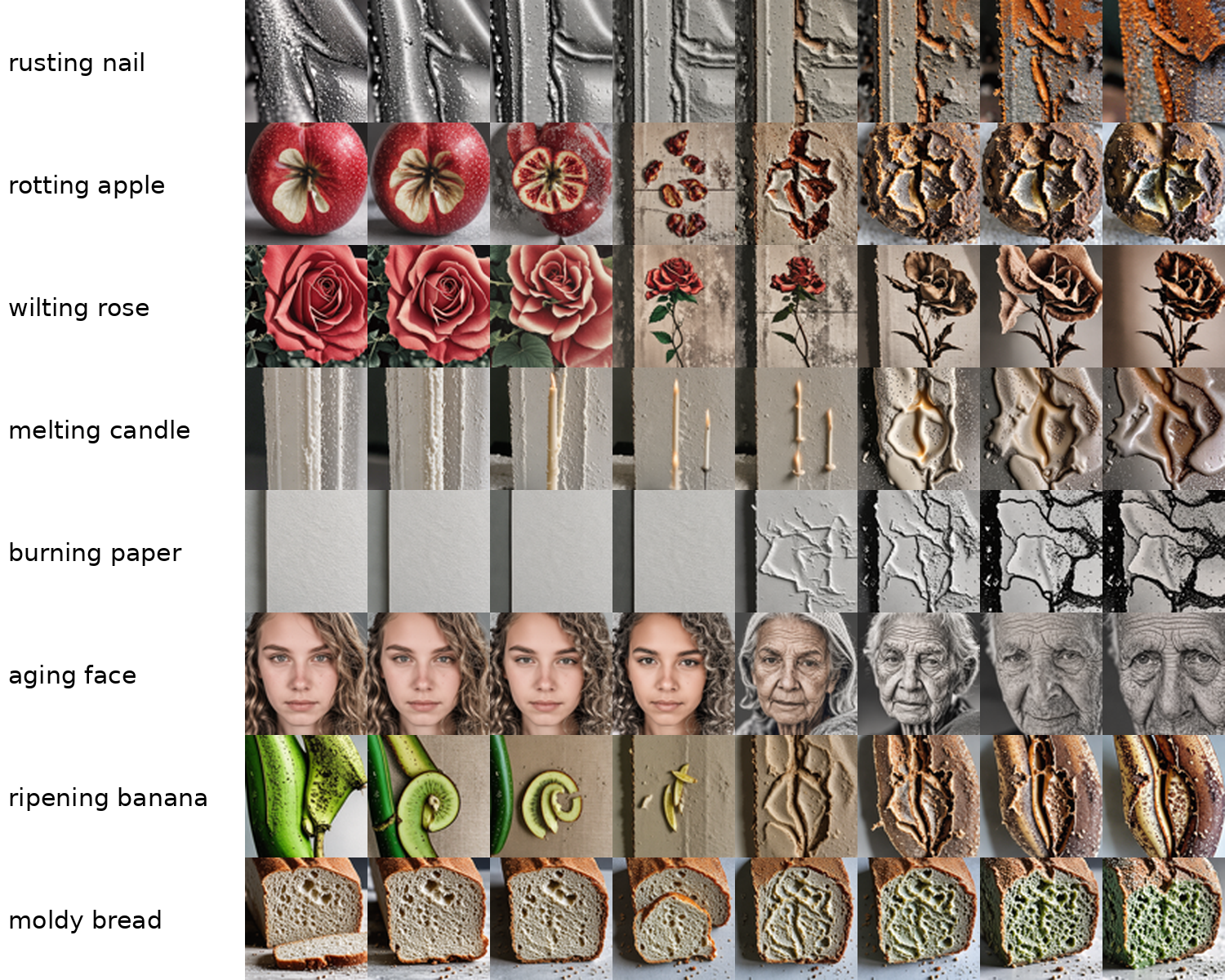}\hfill
\includegraphics[width=0.49\linewidth]{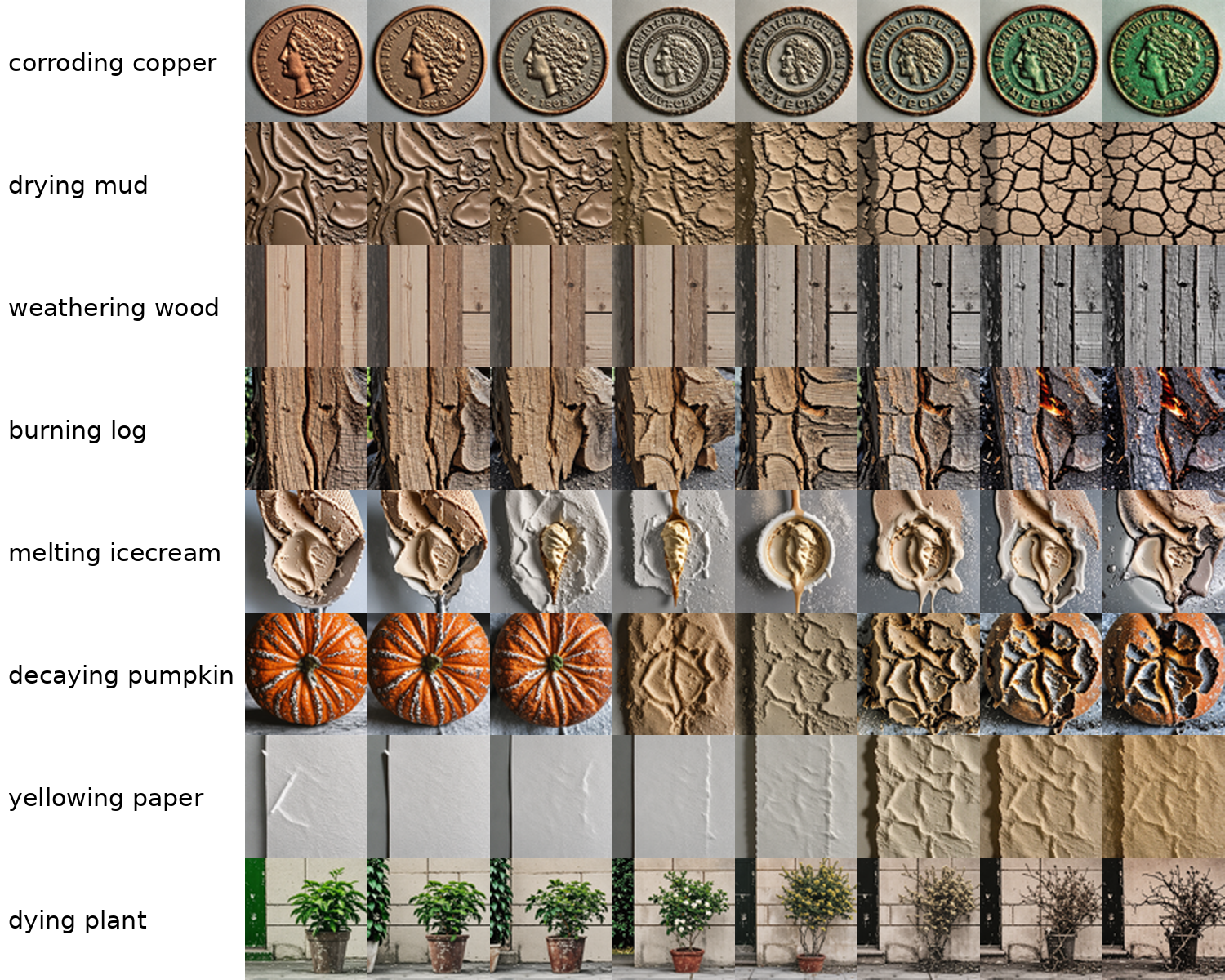}
\caption{\textbf{Gallery of generated graded irreversible processes.} Each row is a process; columns
interpolate the attribute from start to end state at fixed nuisance. Sixteen diverse processes across
materials (rust, corrosion, weathering, char), biology (rot, mold, ripening, decay, plant death), and
a person (face aging) all render as coherent monotone progressions.}
\label{fig:gallery}
\end{figure*}

\begin{figure}[t]
\centering
\includegraphics[width=\linewidth]{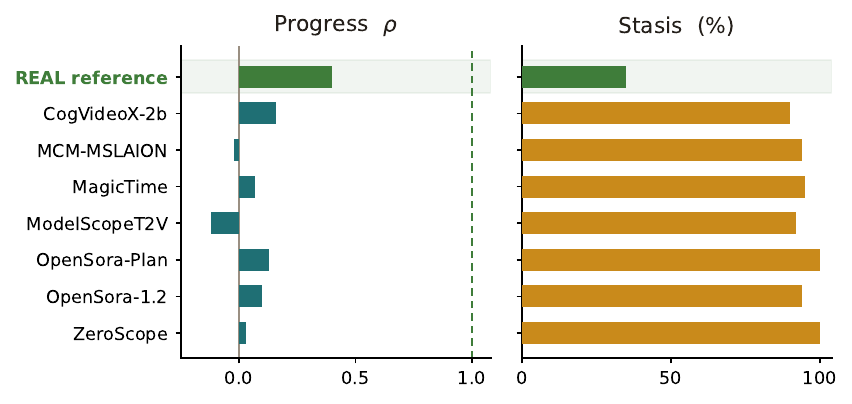}
\caption{\textbf{The progress--stasis protocol across real footage and seven T2V models.} Per-model means
of the two null-surviving metrics. Real reference footage (green), run through the identical protocol, is
the achievable baseline ($\rho{=}{+}0.40$, $35\%$ static); every generator instead clusters near zero
progress and near-total stasis, clearly separated from real footage.}
\label{fig:aot}
\end{figure}

Let $x_{1:T}$ be a generated clip and let $a\colon \mathcal{X}\to\mathbb{R}$ denote a scalar
\emph{irreversible attribute} (rust, rot, char, wilt, melt fraction, age). We say the clip
\emph{violates irreversibility} at $t$ if $a(x_{t}) < a(x_{t-1})-\varepsilon$, and we define the
\emph{violation rate} $V = \frac{1}{T-1}\sum_{t} \mathbf{1}[a(x_t)<a(x_{t-1})-\varepsilon]$. The goal
is a generator whose clips have $V\approx 0$ for irreversible attributes \emph{without} sacrificing
generation quality, controllability, or temporal identity.

\paragraph{Measuring $V$ honestly.} A subtlety that pervades this problem: if $V$ is computed with
the same readout that a method optimizes, the measurement is confounded. We therefore always report
$V$ under a \emph{readout-independent} probe $\tilde a$ (a physical/color statistic or a held-out
estimator disjoint from any guidance signal). This separation is not a detail---it is what
distinguishes a real correction from a gamed one. We keep three notions distinct: a
\emph{ground-truth attribute} (available only in the controlled renderer), a \emph{held-out
attribute probe} (a statistic not used in optimization; on SD and real video it is itself a proxy,
not ground truth), and \emph{human evaluation} (a three-annotator study we run below to validate the
readouts). ``Readout-independent'' means \emph{not used in the optimization}, not \emph{equal to the true attribute}.

\paragraph{Violation must be measured against progress, with attribute-specific axes.} A second
subtlety is easy to get wrong, and we establish it at scale (Sec.~\ref{app:chrono}, $1{,}050$ clips
from ChronoMagic-Bench \cite{chronomagic}): a \emph{generic}, caption-free monotonicity metric---for
instance requiring the CLIP-embedding distance from the initial frame to be non-decreasing---is
\textbf{invalid}. It rewards \emph{doing nothing}: a static or blurry generation has almost no
frame-to-frame change and therefore almost no ``violations,'' while genuine time-lapse fluctuates
(lighting, camera, texture) and scores \emph{worse}. Empirically, real reference videos violate more
($0.29$) than six real T2V models ($0.11$--$0.18$), and the honest separation is elsewhere: generated
clips make $2\!-\!3\times$ less progress and up to $23\%$ of them barely move at all. We therefore
measure $V$ with \emph{attribute-specific} directional readouts and independent probes throughout,
and always alongside the progress actually made.

\paragraph{The violation rate itself is not a usable statistic (a negative result about our own
metric).} A per-clip normalized violation rate is inadmissible, as null baselines show (Table~\ref{tab:null}): under per-clip
normalization a readout that is \emph{pure noise} yields $V=0.498\pm0.076$, and even a readout with a
\emph{strong genuine monotone drift} plus noise yields only $V=0.444\pm0.076$. The statistic
saturates near $0.45$--$0.50$ for everything, because normalization stretches whatever noise exists
to full range; it cannot separate a reversing generator from a noisy readout. This is the mirror image
of the failure in Sec.~\ref{app:chrono}: a generic monotonicity score \emph{rewards} stasis, whereas a
per-clip-normalized one \emph{punishes} it to chance. Neither is admissible, and we therefore report
no headline $V$ for real T2V.

\paragraph{Large-scale confirmation on ChronoMagic-Bench.}
\label{app:chrono}

We test the measurement protocol at scale on ChronoMagic-Bench \cite{chronomagic}, comparing
$150$ clips each from the \textbf{real} time-lapse reference set and from \textbf{six} real
text-to-video models (ModelScopeT2V, ZeroScope, MagicTime, MCM-MSLAION, OpenSora-1.2,
OpenSoraPlan-1.2)---$1{,}050$ clips in total, with $95\%$ confidence intervals.

\paragraph{A generic metric is invalid (negative result).} Define the caption-free progress signal
$d_t=1-\cos(\phi(x_t),\phi(x_0))$ (distance from the initial state), which an irreversible process
should never decrease. Under this metric the \emph{real} videos look \emph{worse} than the generated
ones: real $V=0.293\pm0.023$ versus $0.11$--$0.18$ for the six models. The reason is structural: the
metric rewards stasis. Generated time-lapse is markedly more static, so its $d_t$ is flat and cannot
``violate'', whereas real footage genuinely changes and additionally fluctuates with lighting,
camera and texture. Normalizing by the motion actually made---the backtracking ratio
$R=\sum_t\max(0,-\Delta d_t)/\sum_t|\Delta d_t|$, evaluated only on clips with non-trivial
progress---does not rescue it either (real $R=0.261\pm0.018$; generated $R=0.20$--$0.26$).

\paragraph{What the data do show.} The honest separation is in \emph{progress}, not backtracking
(Fig.~\ref{fig:chrono}): real clips move $0.164$ from their initial state on average, generated clips
only $0.049$--$0.089$ ($2\!-\!3\times$ less), and $8$--$23\%$ of generated clips make essentially no
progress at all (versus $2\%$ for real). Generated time-lapse fails less by reversing a well-formed
attribute than by never developing one.

\begin{figure}[t]
\centering
\includegraphics[width=\linewidth]{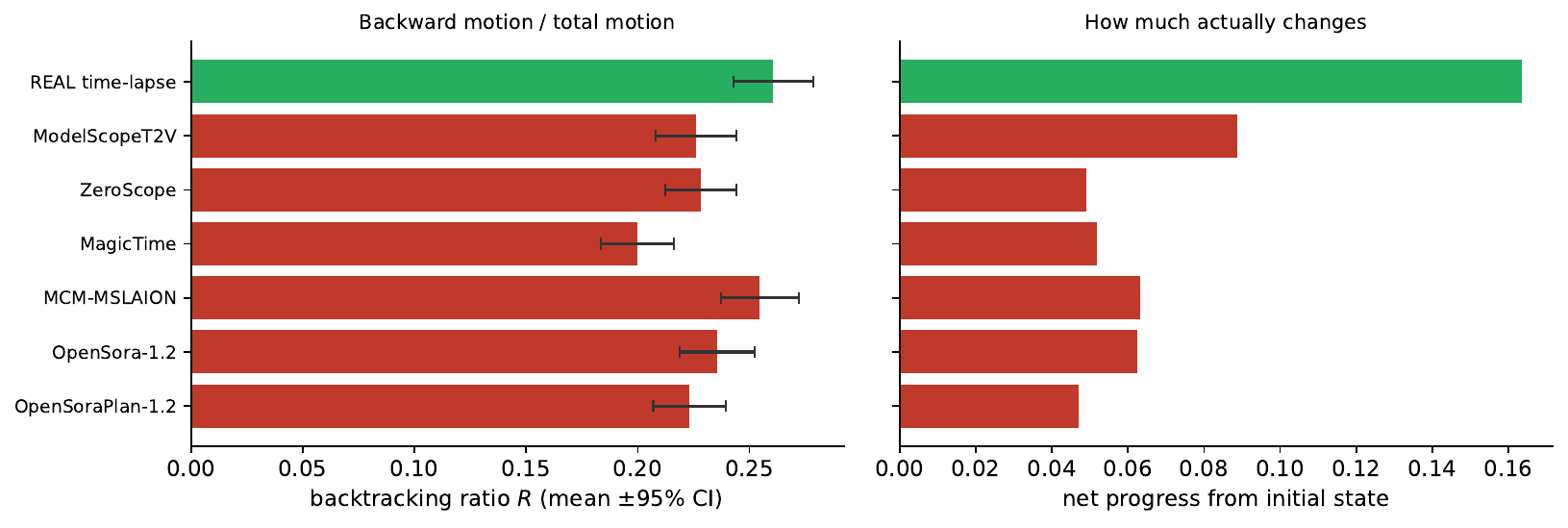}
\caption{\textbf{A generic monotonicity metric does not separate real from generated.} Left:
backtracking ratio (backward motion / total motion) is \emph{no lower} for six T2V models than for
real time-lapse. Right: the models simply change $2$--$3\times$ less. A generic metric therefore
rewards stasis, which is why we use attribute-specific readouts, independent probes, and always
report progress alongside violations.}
\label{fig:chrono}
\end{figure}

\paragraph{Implication for the protocol.} Irreversibility cannot be audited with a generic embedding
monotonicity score: (i) any monotonicity metric assigns a perfect score to a constant signal, so it
must be reported alongside---or conditioned on---the progress actually made; and (ii) the readout
must be \emph{attribute-specific}, as in our per-prompt directional axes (Sec.~\ref{sec:setup}), since
a generic embedding distance mostly measures appearance fluctuation. This is why every violation
number in this paper is computed with an attribute-specific readout and an independent probe.

\begin{table}[t]
\caption{Null baselines for a per-clip-normalized violation rate ($T{=}17$, $2000$ draws). The
statistic is uninformative: it returns $\approx0.45$--$0.50$ whether the readout is pure noise or a
strongly monotone signal. We consequently do not report it.}
\label{tab:null}
\begin{center}\small
\resizebox{\columnwidth}{!}{%
\begin{tabular}{lcccc}
\toprule
\textbf{Readout} & $\epsilon{=}0$ & $\epsilon{=}0.003$ & $\epsilon{=}0.01$ & $\epsilon{=}0.05$ \\
\midrule
pure Gaussian noise            & $0.502$ & $0.498$ & $0.490$ & $0.452$ \\
random walk (no drift)         & $0.498$ & $0.492$ & $0.478$ & $0.396$ \\
noise $+$ weak true drift      & $0.497$ & $0.497$ & $0.487$ & $0.447$ \\
noise $+$ \emph{strong} true drift & $0.448$ & $0.444$ & $0.435$ & $0.381$ \\
\bottomrule
\end{tabular}}
\end{center}
\end{table}

\paragraph{The readout tracks progress.}
Before enforcement, we verify that a scalar attribute readout can \emph{measure} progress on
generated content. Using text-embedding interpolation to synthesize graded sequences (fixed latent,
attribute prompt interpolated $0\!\to\!1$) for five processes, an independent CLIP directional
readout recovers the ground-truth progress with median rank correlation $\approx0.84$ (per-process
$0.66$--$0.92$), detects injected reversals in $100\%$ of cases, and is stable across paraphrased
attribute axes (rank correlation $0.73$--$0.94$). The weakest case (burning, $0.66$) is the abrupt
process noted in the scope (Fig.~\ref{fig:measurement}). These results support using such a readout
to \emph{detect} violations (and to define the independent evaluation probe), while
Section~\ref{sec:gaming} shows it must \emph{not} be used as a guidance target.

\begin{figure}[t]
\centering
\includegraphics[width=0.82\linewidth]{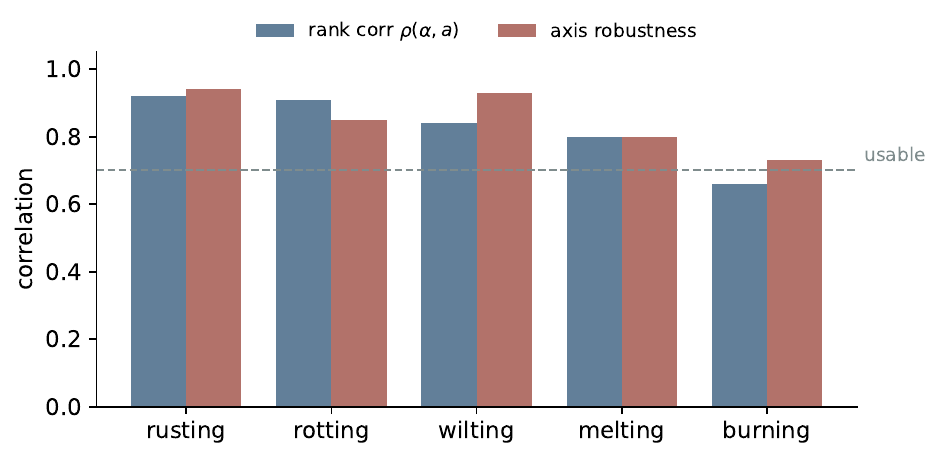}
\caption{\textbf{Readout measurement study.} Rank correlation between a directional attribute readout
and ground-truth progress, and stability across paraphrased attribute axes, for five processes on
generated content. Gradual attributes are tracked well; the abrupt process (burning) is weakest,
matching the stated scope. Injected reversals are detected in $100\%$ of cases.}
\label{fig:measurement}
\end{figure}

\paragraph{Which statistics survive the null.} We null-test candidate metrics, not only the one we reject
(Table~\ref{tab:null2}, $T{=}8$, $20$k draws). Progress $\rho$ is centered at $0$ under every null, so
$\rho\!\gg\!0$ is meaningful. Stasis is well-defined, but its value tracks the readout's noise scale (a
static readout at s.d.\ $0.01$ reads $99\%$ static, at $0.02$ only $36\%$), so it is interpretable only
relative to a matched real-video baseline---which is exactly why we run real footage through the identical
protocol (Table~\ref{tab:sixmodel}). A variance-normalized reversal residual, by contrast, is
null-degenerate like $V$ (pure noise gives $0.85$), so no reversal statistic enters the protocol.

\begin{table}[t]
\caption{Null baselines ($T{=}8$, $20$k draws). Progress $\rho$ is null-centered (safe); stasis depends on
the readout noise scale (read against a matched real baseline); a normalized reversal residual is
null-degenerate ($\approx0.85$ on noise), so we exclude reversal statistics from the protocol.}
\label{tab:null2}
\begin{center}\small
\resizebox{\columnwidth}{!}{%
\begin{tabular}{lccc}
\toprule
\textbf{Null readout} & \textbf{progress $\rho$} & \textbf{local-rev $r$} & \textbf{stasis\% ($\tau{=}.05$)} \\
\midrule
pure Gaussian noise            & $0.00\pm0.38$ & $0.85\pm0.16$ & $0\%$ \\
static $+$ noise (s.d.\ $0.02$)  & $0.00\pm0.38$ & $0.85\pm0.16$ & $36\%$ \\
static $+$ noise (s.d.\ $0.01$)  & $0.00\pm0.38$ & $0.85\pm0.17$ & $99\%$ \\
\bottomrule
\end{tabular}}
\end{center}
\end{table}

\paragraph{What survives, honestly: inconsistency and stasis, not systematic reversal.} A
statistic that is \emph{not} degenerate under the null is the progress correlation $\rho(t,a_t)$
($\rho>0$ for a genuinely irreversible process, $\approx0$ for noise). On single clips five of twelve
prompts had $\rho<0$; but deepening this to \textbf{eight seeds} per process (CogVideoX-2b,
Fig.~\ref{fig:realviol}) shows the single-clip drama was largely variance. Per-process means are
mostly \emph{near zero} with large spread---ice $+0.21\pm0.54$, rusting $+0.25\pm0.36$, wilting
$-0.13\pm0.48$, molding $+0.40\pm0.41$, burning $+0.07\pm0.33$---and $25$--$50\%$ of seeds run
backwards. Crucially the progress \emph{magnitude} is negligible ($0.002$--$0.012$): the generator
neither reliably advances nor reliably reverses the attribute---it barely develops it, and does so
inconsistently. This aligns with the ChronoMagic finding (Sec.~\ref{app:chrono}) that the dominant
failure is stasis. We state the claim at this strength: current generators do not respect
irreversibility, failing by \emph{under-development and seed-to-seed inconsistency}, with a sizeable
minority of outright reversals---not by systematic time-reversal.

\paragraph{Why not a reversal metric?} A natural repair of $V$ is to measure \emph{local} reversal
directly---the normalized distance to the monotone cone, i.e.\ the isotonic residual
$r=\lVert a-\Pi_{\uparrow}a\rVert_2/\lVert a-\bar a\rVert_2$. But this is null-degenerate too: on pure
noise $r=0.85$ (Table~\ref{tab:null2}), because when the attribute barely moves the numerator and
denominator are both at the noise floor, so a high $r$ certifies nothing. Both the community's $V$-style
scores and this variance-normalized residual thus fail null-testing---normalized reversal metrics are
generically uninformative in the near-static regime real generators occupy. We therefore build the
protocol from the two quantities that \emph{survive} the null: progress $\rho$ (null-centered at $0$) and
the stasis rate (read against a matched real baseline). This is the whole protocol; we report no reversal
metric. We are explicit about the consequence: the protocol measures \emph{development}---whether the
attribute advances, and whether the clip moves at all---and \emph{not} illegitimate local reversal, which
we have argued is not measurable by any normalized statistic in this regime. Accordingly we claim
\emph{under-development}, a checkable and human-validatable property, and we do \emph{not} claim that
generators uniquely ``run time backwards.''

\paragraph{The diagnosis holds across six models.} To check that this is not a
CogVideoX artifact, we apply the same attribute-specific protocol to the generated videos of
\textbf{six} open-source T2V models released with ChronoMagic-Bench (MCM-MSLAION, MagicTime,
ModelScopeT2V, OpenSora-Plan-V1.2, OpenSora-1.2, ZeroScope), over nine monotone-attribute categories
(rusting, ice-melting, melting, fruit-rotting, plant-dying, candle-burning, ripening, water-freezing,
construction) with a per-category CLIP directional axis ($\approx560$ generated clips, $108$ real; Table~\ref{tab:sixmodel},
visualized together with CogVideoX-2b in Fig.~\ref{fig:aot}).
Crucially, we run the \emph{real reference videos through the identical protocol} to give the ``ideal''
line of Fig.~\ref{fig:aot} an empirical, achievable baseline---and the separation is clear
(Table~\ref{tab:sixmodel}, bootstrap $95\%$ CIs over clips). Real footage advances
($\rho=+0.40\!\pm\!0.09$) while every generator sits near zero ($\rho\in[-0.12,+0.13]$; two negative);
real footage moves ($35\%$ stasis) while generators are near-static ($92$--$100\%$). The gap is robust to
the threshold: sweeping $\tau\in\{.03,.05,.08,.10,.12\}$ the real/generated stasis is
$15/35/60/82/90\%$ vs.\ $74/96/99/99/100\%$---at every $\tau$ generations are far more static than real
footage. Because clips within a category are not independent, we also test at the \emph{category} level
(the scientific unit): over the nine categories, paired real-vs-generated differences are significant for
both metrics---progress $+0.40$ vs.\ $+0.03$ (paired $t_8{=}4.4$) and stasis $35\%$ vs.\ $96\%$
($t_8{=}{-}8.0$)---and a leave-one-category-out check keeps the progress separation in $[+0.33,+0.43]$
(real progress is positive in $8/9$ categories; water-freezing is the lone exception, where the
liquid$\to$ice CLIP axis is weak). The separation is therefore not an artifact of treating clips as
independent.

\paragraph{Human validation.} Nine annotators rated \emph{how much each process advances} (0--4) on $60$
clips (a shuffled, source-hidden mix of $24$ real and $36$ generated). The CLIP progress tracks human
judgment ($r=0.46$--$0.58$ per annotator against the human rating), and \emph{every} annotator rates real
footage above generated: mean $2.75$ vs.\ $0.99$ on the $0$--$4$ scale, inter-rater agreement $0.73$.
Humans thus confirm the headline: generators under-develop the attribute relative to real footage.
Consistently with under-development rather than reversal, annotators flagged \emph{backwards} motion at
equally low rates for real and generated ($0.21$ vs.\ $0.16$). This is our diagnosis's attribute-specific, human-validated counterpart to the
metamorphic-amplitude finding of ChronoMagic-Bench \cite{chronomagic}; the distinct contribution is the
\emph{null-checked, decomposed} protocol, and above all the demonstration (point (i)) that normalized
reversal metrics are generically null-degenerate.

\begin{table}[t]
\caption{\textbf{The progress--stasis protocol across real footage and seven T2V models} (ChronoMagic-Bench,
nine monotone-attribute categories, per-category CLIP axis; bootstrap $95\%$ CIs over clips). Real
reference footage separates cleanly on both metrics; every generator clusters near zero progress and
near-total stasis. CogVideoX-2b$^{*}$ is the five-process $\times$ eight-seed mean.}
\label{tab:sixmodel}
\begin{center}\small
\resizebox{\columnwidth}{!}{%
\begin{tabular}{lccc}
\toprule
\textbf{Source} & \textbf{progress $\rho$} & \textbf{stasis\% ($\tau{=}.05$)} & \textbf{clips} \\
\midrule
\textbf{REAL reference}   & $\mathbf{+0.40\pm0.09}$ & $\mathbf{35\%}$ & $108$ \\
\midrule
CogVideoX-2b$^{*}$ & $+0.16$ & $90\%$ & $40$ \\
MCM-MSLAION       & $-0.02\pm0.11$ & $94\%$ & $108$ \\
MagicTime         & $+0.07\pm0.11$ & $95\%$ & $108$ \\
ModelScopeT2V     & $-0.12\pm0.11$ & $92\%$ & $108$ \\
OpenSora-Plan-V1.2& $+0.13\pm0.24$ & $100\%$ & $16$ \\
OpenSora-1.2      & $+0.10\pm0.10$ & $94\%$ & $108$ \\
ZeroScope         & $+0.03\pm0.11$ & $100\%$ & $108$ \\
\bottomrule
\end{tabular}}
\end{center}
\end{table}

\begin{figure}[t]
\centering
\includegraphics[width=0.72\linewidth]{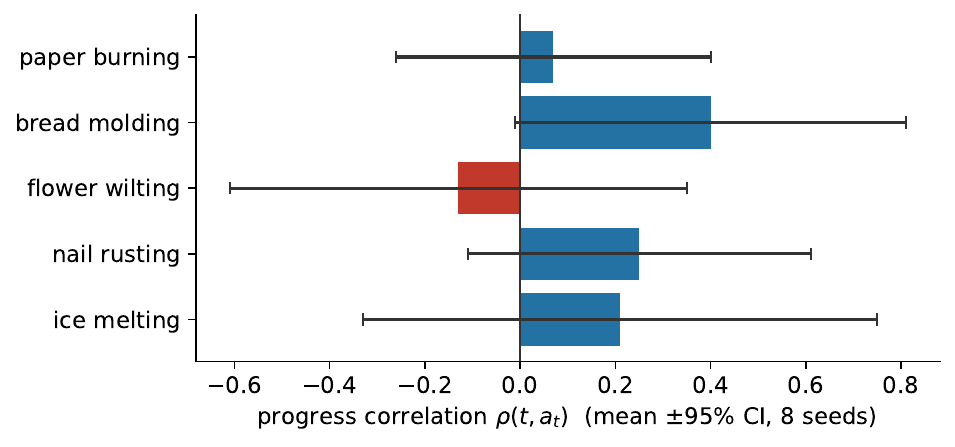}
\caption{\textbf{Multi-seed progress correlation (CogVideoX-2b, 8 seeds/process, mean $\pm$95\% CI).}
Means hover near zero with large spread; $25$--$50\%$ of seeds run backwards and progress magnitudes
are negligible ($0.002$--$0.012$). The failure is under-development and inconsistency, not systematic
reversal. We plot $\rho$, not the degenerate violation rate $V$ (Table~\ref{tab:null}).}
\label{fig:realviol}
\end{figure}

\section{Post-hoc frozen-readout optimization is adversarially gamed}
\label{sec:gaming}

Consider the natural inference-time approach. Given a differentiable readout
$a_t=\langle \phi(x_t), u\rangle$ (here $\phi$ a CLIP image encoder \cite{clip} and $u$ a
text-defined attribute direction), enforce monotonicity by isotonic projection of $\{a_t\}$ to a
non-decreasing target $\hat a_t$ and steer each frame's latent by gradient descent on
$(a_t-\hat a_t)^2 + \lambda\lVert z_t - z_t^{\mathrm{base}}\rVert^2$, backpropagating through the
decoder and $\phi$. This is the video analogue of classifier / reconstruction guidance and requires
no training.

\paragraph{It does not add the attribute; it fools the readout.} We stress-test with a controlled
injected reversal: within a graded sequence (a nail progressively rusting) we replace a late frame's
latent with an early, low-attribute one, producing a large, visible reversal, and then run the
guidance to pull its readout up to the monotone target. Across guidance strengths, the readout rises
to its target ($a\!:\,{+}0.03\!\to\!{+}0.11$), but the readout-\emph{independent} color probe (mean
orange–blue chroma, a physical proxy for rust) stays flat ($0.009\!\to\!0.015$ versus $0.19$ for a
truly rusted frame), and the frame remains visibly clean metal (Fig.~\ref{fig:gaming}). The
optimization has found readout-satisfying but non-semantic changes. Quantitatively
(Fig.~\ref{fig:gamingbar}), guidance closes $62\%$ of the \emph{readout} gap to the real rusted state
but only $3\%$ of the \emph{independent-probe} gap. It also shows the failure \emph{corrupts
evaluation} whenever the same readout scores success, motivating our independent-probe protocol.

\paragraph{Scope: post-hoc optimization, not classifier guidance.} One might object that optimizing a
generated frame's latent is an attack, not guidance, since true classifier/energy guidance injects the
gradient at every reverse step where the denoiser re-projects onto the manifold. We tested an in-loop
variant: the prior does block the adversarial solution (readout $0.32$ of target vs.\ $1.00$ post-hoc),
but the independent probe still does not move ($0.02$)---it neither games the readout nor produces the
attribute. We now report the sweep (Fig.~\ref{fig:gsweep}): six guidance strengths $\times$ two step budgets $\times$
three readouts. The picture is nuanced. With a \emph{CLIP} readout, raising the strength drives the
readout to $0.98$ while the independent probe barely moves ($\le0.38$) and fidelity collapses
(LPIPS $0.61$)---the image is destroyed, not transformed. With a \emph{learned} attribute regressor,
strong in-loop guidance \emph{does} strongly increase the held-out probe (up to $1.7$, unlike
post-hoc optimization which never did)---but only with large perceptual deviation from the plain
generation (LPIPS $\ge0.83$); whether that deviation is the genuine attribute or over-saturation we
cannot certify without human evaluation.
Across all $36$ settings the corner with high probe recovery \emph{and} small perceptual deviation is
empty. We are careful here: LPIPS penalizes \emph{any} change, so genuinely adding the attribute
\emph{must} incur some deviation---an empty small-deviation corner is therefore weaker evidence than it
first appears, and consistent with the human study below, which finds strong guidance perceptibly
directional. The claim we stand behind is only the milder one: readout-guided generation exhibits an
attribute--fidelity/identity trade-off (CLIP guidance games; learned-readout guidance raises the probe
only with large deviation), which enforcing the attribute \emph{by construction} sidesteps. We accordingly restrict our negative claim to
\emph{post-hoc readout optimization}, and treat properly-tuned on-manifold guidance (and online /
adversarially-retrained reward models, which we do not test) as open.

\paragraph{A blind human study: guidance moves the attribute perceptibly, but only partially.} The
sweep leaves one question the probes cannot settle---whether strong in-loop guidance produces a
\emph{perceptible} attribute or merely over-saturates. We ran a blind study to check. For six random
seeds we rendered three images from the \emph{same} noise: an injected low-attribute frame ($L$, clean
metal), the strong in-loop CLIP-guidance output ($G$), and a genuinely high-attribute generation ($R$,
real reference). Within each case the three were shuffled, and \textbf{four} annotators, blind to
identity, independently ranked them by rust and flagged same-object pairs and artifacts
($4\times6=24$ judgments; mean pairwise agreement on the full ordering $92\%$). The results
(Table~\ref{tab:human}) are consistent. The real reference $R$ was judged \emph{most} rusted in
$24/24$ judgments, and guidance $G$ was judged \emph{more rusted than the clean baseline} $L$ in
$19/24$ ($79\%$; majority ordering $R\!>\!G\!>\!L$ in five of six cases): in-loop guidance is therefore
\emph{not} pure readout-gaming at the human level---it moves the appearance toward rust perceptibly.
The sharp limit is that $G$ \emph{never} reached $R$'s level ($R\!>\!G$ in $24/24$)---the attribute is
only \emph{partially} produced. On \emph{identity} the raters split: $G$ was judged the same physical
object as the frame it edited in $9/24$ ($38\%$, and $0$--$3$ of six per rater). This is close to chance
and inconsistent across raters, so we read it not as a positive ``ambiguity'' finding but as
\emph{annotators lacking a firm basis to judge object identity} from a shuffled filmstrip---i.e.\ the
study is simply underpowered for the identity question (it does correct a single-annotator pilot that had
over-concluded identity was always broken). Notably $G$
was \emph{not} systematically artifact-laden ($5/24$; the clean low-attribute renders drew far more
artifact flags, $15/24$). The honest reading: strong guidance is perceptibly directional but
\emph{partial}, with ambiguous identity preservation---which is exactly the regime that by-construction
enforcement is meant to make moot. To avoid pseudo-replication we also aggregate at the level of the
\emph{six cases} (the true independent unit; the four raters are repeated measures): a majority of raters
judged $G\!>\!L$ in $5/6$ cases, $R\!>\!G$ in $6/6$, and $L{=}G$ (identity preserved) in only $2/6$---the
same conclusions as the pooled counts. Scope: $n{=}6$ cases and four annotators on \emph{rust only},
evaluating \emph{in-loop guidance} (not the structural method on natural video); the independent probe
is a chroma proxy, not a certificate of physical rust---so this is a small, inter-rater-consistent
sanity check, not a powered human evaluation of the method.

\begin{table}[t]
\caption{Blind human study of strong in-loop guidance ($6$ cases $\times$ $4$ annotators $=24$
judgments; $92\%$ pairwise agreement on the full ordering). $L$: injected clean baseline; $G$: guidance
output; $R$: real high-attribute reference, rendered from identical noise and shuffled. ``$R$ most
rusted'' and ``$G>L$'' count judgments; ``$L{=}G$ same object'' counts identity preservation.}
\label{tab:human}
\begin{center}\small
\resizebox{\columnwidth}{!}{%
\begin{tabular}{lc}
\toprule
\textbf{Judgment} & \textbf{judgments (of 24)} \\
\midrule
Real reference $R$ ranked most rusted & $24/24$ \\
Guidance $G$ more rusted than clean baseline $L$ & $19/24$ ($79\%$) \\
Guidance $G$ reaches the real reference $R$ & $0/24$ \\
$G$ preserves the edited object's identity ($L{=}G$) & $9/24$ ($38\%$) \\
$G$ flagged for artifacts ($L$ flagged: $15/24$) & $5/24$ \\
\bottomrule
\end{tabular}}
\end{center}
\end{table}

\paragraph{Post-hoc optimization: not one CLIP score---but not universal either.} To test whether this is an
artifact of a single readout, we attack \textbf{five} frozen readouts with the same latent-space
guidance on the same injected reversals, scored by the same independent probe
(Fig.~\ref{fig:multireadout}): the original CLIP directional score, a paraphrased CLIP axis, a
\emph{learned} attribute regressor (a CNN trained on (frame, $\alpha$) pairs), a
\emph{noise-augmented, more robust} regressor, and an \emph{ensemble} of three. Every one is driven
to exactly its target (readout gain $1.00$), but the true attribute barely moves: the CLIP scores are
fully gamed (independent-probe gain $0.02$ and $0.00$), the ensemble is nearly so ($0.23$), and the
learned and robust regressors resist \emph{best but still only partially} ($0.32$ and $0.26$, i.e.\
$\approx15\times$ better than CLIP yet leaving $\ge\!68\%$ of the attribute gap unfilled while
reporting a perfect score). We draw the honest conclusion: readout choice \emph{matters}---a learned
attribute model is markedly harder to game than a VLM similarity---but none of the five is faithful
enough to certify the attribute, and an ensemble of unfaithful readouts inherits their weakness.
Combined with Proposition~1, this supports the claim that faithfulness must be
\emph{established} rather than assumed; it does not prove that \emph{every} conceivable robust or
on-manifold-constrained readout fails.

\begin{figure}[t]
\centering
\includegraphics[width=0.82\linewidth]{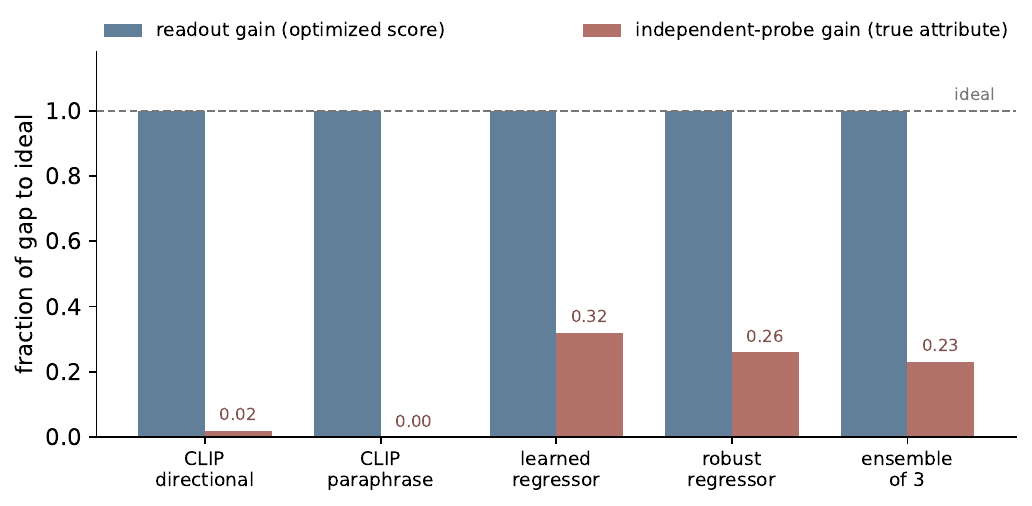}
\caption{\textbf{Five frozen readouts, one guidance procedure, one independent probe.} All readouts
reach their target (purple $=1.0$), but the readout-independent attribute barely moves (green).
Learned/robust attribute regressors resist best ($0.26$--$0.32$) yet still leave most of the gap
unfilled; CLIP-based scores and their ensemble are gamed outright.}
\label{fig:multireadout}
\end{figure}

A non-gradient variant---reject/resample: draw fresh samples until one is monotone---does yield a
\emph{real} correction (the color probe rises to a genuinely rusted level), but only by abandoning
the frame's content (a fresh sample is a \emph{different} object, LPIPS $0.72$ to the original) and
at a multiplicative sampling cost. Neither guidance nor resampling is a satisfactory answer:
guidance is fake, resampling breaks identity and is expensive.

\begin{figure}[t]
\centering
\includegraphics[width=0.92\linewidth]{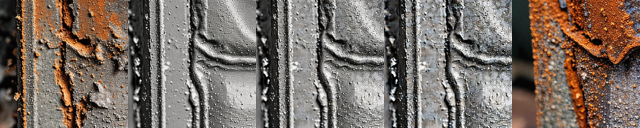}
\caption{\textbf{Post-hoc frozen-readout guidance is gamed (for the readouts we test).} Left to right: a rusted neighbour frame; an
injected low-attribute (clean-metal) frame; the frame after weak and strong monotonicity guidance;
another rusted neighbour. Guidance raises the CLIP readout to the monotone target, but the frames
stay visibly clean metal and the readout-independent color probe does not move---the readout is
satisfied without producing rust.}
\label{fig:gaming}
\end{figure}

\begin{figure}[t]
\centering
\includegraphics[width=0.82\linewidth]{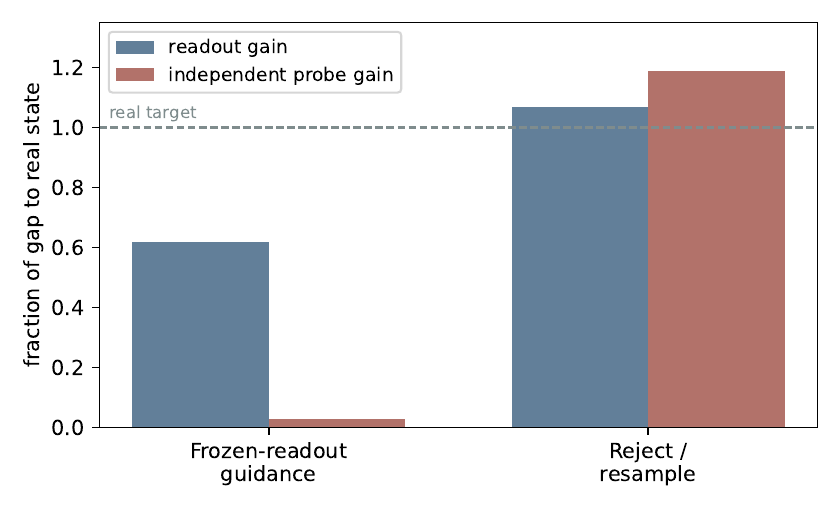}
\caption{\textbf{Guidance games the readout; a real fix must move the independent probe.} Fraction
of the gap to the true rusted state closed by each method, measured by the guidance \emph{readout}
(purple) versus a \emph{readout-independent} physical probe (green). Frozen-readout guidance closes
most of the readout gap but essentially none of the probe gap ($3\%$)---it is gamed. Reject/resample
moves both, but only by replacing the object's identity.}
\label{fig:gamingbar}
\end{figure}

\begin{figure}[t]
\centering
\includegraphics[width=0.82\linewidth]{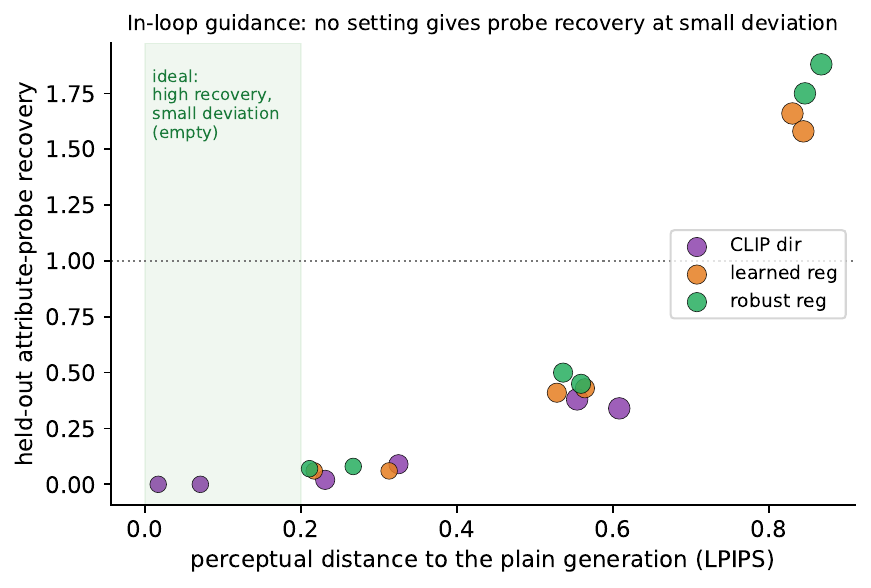}
\caption{\textbf{In-loop diffusion guidance sweep (held-out probe recovery vs.\ perceptual deviation).} Six strengths $\times$
two step budgets $\times$ three readouts; point size is guidance strength. CLIP guidance (purple)
barely moves the held-out probe while the output deviates strongly from the plain generation;
learned/robust regressors (orange/green) raise the probe only at large perceptual deviation. The
corner with high probe recovery \emph{and} small deviation is empty for all $36$ settings. The probe
is a proxy; we do not claim it equals the true attribute.}
\label{fig:gsweep}
\end{figure}

\section{Monotone-by-construction in a disentangled attribute latent}
\label{sec:method}

\begin{figure*}[t]
\centering
\includegraphics[width=0.95\textwidth]{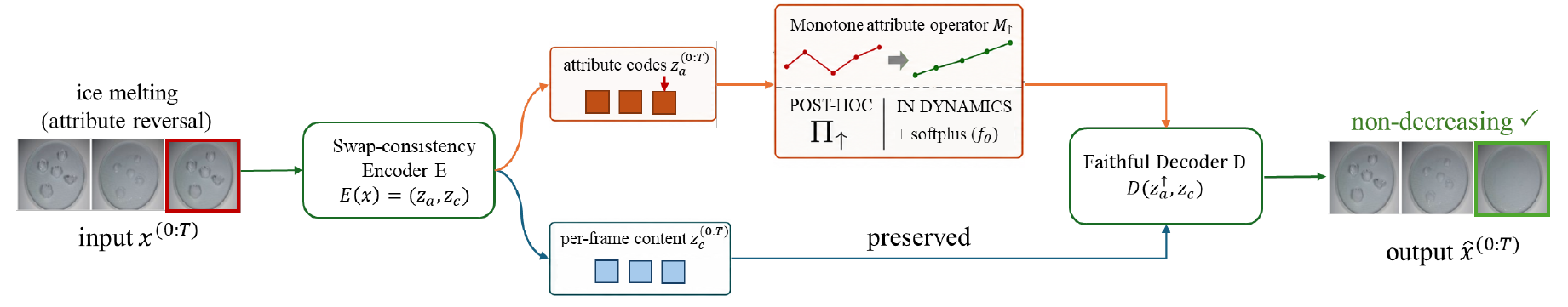}
\caption{\textbf{Method: enforce the arrow of time in a disentangled latent, not on a readout.} A
swap-consistency encoder $E$ splits each frame into an attribute $z_a$ and content $z_c$. A reversal in
the attribute trajectory (red) is removed by a monotone constraint (isotonic projection $\Pi_\uparrow$
post-hoc, or a softplus positive-increment dynamics), giving a non-decreasing path (green) while $z_c$ is
untouched; the decoder $D$ then renders $(z_a^{\uparrow},z_c)$. Enforcing the order in the rendered latent
leaves nothing to game, provided $E$ is faithful---which holds on controlled data but not yet on real
video.}
\label{fig:method}
\end{figure*}

The lesson of Section~\ref{sec:gaming} is to make irreversibility a property of the parameterization
rather than of an external score. We learn an encoder $E(x)=(\za,\zc)$ and decoder $D(\za,\zc)$ with
$\za\in\mathbb{R}$ the irreversible attribute and $\zc$ the nuisance, such that $D$ is sensitive to $\za$
while $\zc$ carries no attribute information; irreversibility is then enforced structurally by
constraining the trajectory of $\za$ to be non-decreasing (a positive-increment dynamics during
generation, or post-hoc isotonic projection of the encoded $\{\za^{(t)}\}$). Because the constrained
coordinate is one the decoder faithfully renders, raising $\za$ \emph{is} adding the attribute, with no
adversarial slack. The construction is only as good as the disentanglement: if the attribute leaks into
$\zc$, projecting $\za$ changes nothing. We obtain the split with a \emph{swap} objective---decoding
$D(\za^{(i)},\zc^{(j)})$ must reconstruct attribute $i$ with nuisance $j$---which forces the attribute
into $\za$ and nuisance into $\zc$ (an adversarial-independence penalty was unstable and collapsed $\za$).

\subsection{Controlled validation}

We isolate the mechanism in a domain with a \emph{known} ground-truth attribute and a
readout-independent probe: a rendered ``rust bar'' whose attribute $a\in[0,1]$ controls the bar's
grey$\to$orange color and rust speckle density, with nuisance factors (bar position, background
shade, width) varied independently; the independent probe is again mean orange–blue chroma. We train
an autoencoder with reconstruction, a swap-disentanglement term, and a light $\za\!\approx\!a$
supervision, then evaluate on a monotone attribute ramp with an \emph{injected} mid-sequence
reversal.

\paragraph{Disentanglement succeeds.} Held-out reconstruction MSE is $1.9\times10^{-3}$ and
$\mathrm{corr}(\za,a)=0.95$. The decode is strongly controlled by $\za$ (probe rises by a gap of
$0.14$ as $\za$ goes $0\!\to\!1$) and nearly \emph{invariant} to $\zc$ (resampling $\zc$ shifts the
probe with std $0.003$), i.e.\ the attribute lives in $\za$ and nuisance in $\zc$
(Fig.~\ref{fig:disentangle}).

\paragraph{Structural monotonicity is real and non-gameable.} Over $8$ seeds and injection positions
(Table~\ref{tab:ablation}), decoding the raw encoded $\za$ reproduces the reversal (violation rate
$0.111$), whereas decoding the monotone-projected $\za$ is exactly monotone (violation rate $0.000$).
Crucially the repair is genuine: the \emph{readout-independent} probe at the repaired frame recovers
$42\%\pm9\%$ of the gap to the neighbouring envelope (Fig.~\ref{fig:traj}), a real attribute
increase---not the $3\%$ that guidance produced---while identity is preserved (frame-to-neighbour MSE
$<10^{-4}$; Fig.~\ref{fig:monotone}). This is exactly the regime in which frozen-readout guidance was
gamed, and here the correction is real because the decoder renders $\za$ into the true attribute.

\paragraph{Baselines, and where the benefit comes from.} Against simple alternatives (Table~\ref{tab:base})
a naive cumulative clamp ($\za^{(t)}\!\leftarrow\!\max(\za^{(t)},\za^{(t-1)})$) already attains $V=0$
and repairs the probe ($0.77$), so on a \emph{single} reversal the $L_2$-optimality of isotonic
regression is not what matters---its value is the minimal, globally-optimal correction under
\emph{multiple} violations (Prop.~5), which we confirm empirically (App.~S5): under $1/2/4$ injected reversals the isotonic correction has strictly smaller total latent modification $\lVert\hat\za-\za\rVert_2$ than clamp ($0.14/0.16/0.18$ vs.\ $0.18/0.24/0.28$), and as violations accumulate it perturbs \emph{non-violated} frames far less ($0.014$ vs.\ $0.025$ at four violations), because clamp propagates an early spike forward while isotonic pools locally; temporal smoothing, a common ``consistency''
fix, does not even remove the violation ($V=0.167$). The decisive control is the last row: we retrain
the autoencoder \emph{without} the swap term, giving an entangled latent (margin $\approx0$); applying
the identical isotonic projection then yields a perfectly \emph{monotone latent} but \emph{zero}
attribute repair ($V=0.111$, recovery $0.00$). Monotonicity of the latent is therefore worthless
without faithful disentanglement---the benefit comes from disentanglement, not from clamping negative
increments.

\begin{table}[t]
\caption{Baselines and the key ablation (synthetic renderer, 8 seeds/positions; independent probe).
On one reversal, clamping already works, so isotonic's optimality matters only under multiple
violations (Prop.~5). Last two rows: without swap disentanglement (margin $\approx0$)
the same isotonic projection gives a monotone latent but no true-attribute repair---the benefit is
disentanglement, not monotonicity.}
\label{tab:base}
\begin{center}\small
\resizebox{\columnwidth}{!}{%
\begin{tabular}{lccc}
\toprule
\textbf{Method / setting} & \textbf{A1 margin} & \textbf{violation rate} $\downarrow$ & \textbf{probe recovery} $\uparrow$ \\
\midrule
unconstrained                        & --      & $0.111$ & $0.00$ \\
temporal smoothing                   & --      & $0.167$ & --     \\
cumulative clamp                     & --      & $0.000$ & $0.77$ \\
isotonic, faithful AE (ours)         & $0.178$ & $0.000$ & $0.53$ \\
\midrule
isotonic, \emph{entangled} AE (no swap) & $\approx0$ & $0.111$ & $\mathbf{0.00}$ \\
\bottomrule
\end{tabular}}
\end{center}
\end{table}

\begin{table*}[t]
\begin{minipage}[t]{0.5\linewidth}\centering
\caption{\textbf{(2a) Controlled structural repair} (synthetic renderer, ground-truth attribute,
8 seeds/positions). Recovery $0.42\!\pm\!0.09$ here vs.\ $0.53$ in Table~\ref{tab:base} (separate runs:
random vs.\ fixed injection position); both show partial ($\approx0.4$--$0.5$) but genuine recovery,
versus $0.03$ for guidance.}
\label{tab:ablation}
\centering\small\setlength{\tabcolsep}{4pt}
\begin{tabular}{lccc}
\toprule
\textbf{Method} & \textbf{$V\downarrow$} & \textbf{recovery} $\uparrow$ & \textbf{identity} \\
\midrule
Unconstrained & $0.111$ & -- & (ref) \\
\textbf{Ours} & $\mathbf{0.000}$ & $\mathbf{0.42}$ & $<10^{-4}$ \\
\bottomrule
\end{tabular}
\end{minipage}\hfill
\hfill\begin{minipage}[t]{0.46\linewidth}\centering
\caption{\textbf{(2b) Guidance vs.\ resampling} stress test (real SD images, held-out probe).}
\label{tab:stress}
\centering\small\setlength{\tabcolsep}{4pt}
\begin{tabular}{lcc}
\toprule
\textbf{Method} & \textbf{recovery} $\uparrow$ & \textbf{identity} \\
\midrule
Guidance & $0.03$ (gamed) & preserved \\
Reject/resample & $\approx1.0$ & broken ($0.72$) \\
\bottomrule
\end{tabular}
\end{minipage}
\end{table*}

\begin{figure}[t]
\centering
\includegraphics[width=0.98\linewidth]{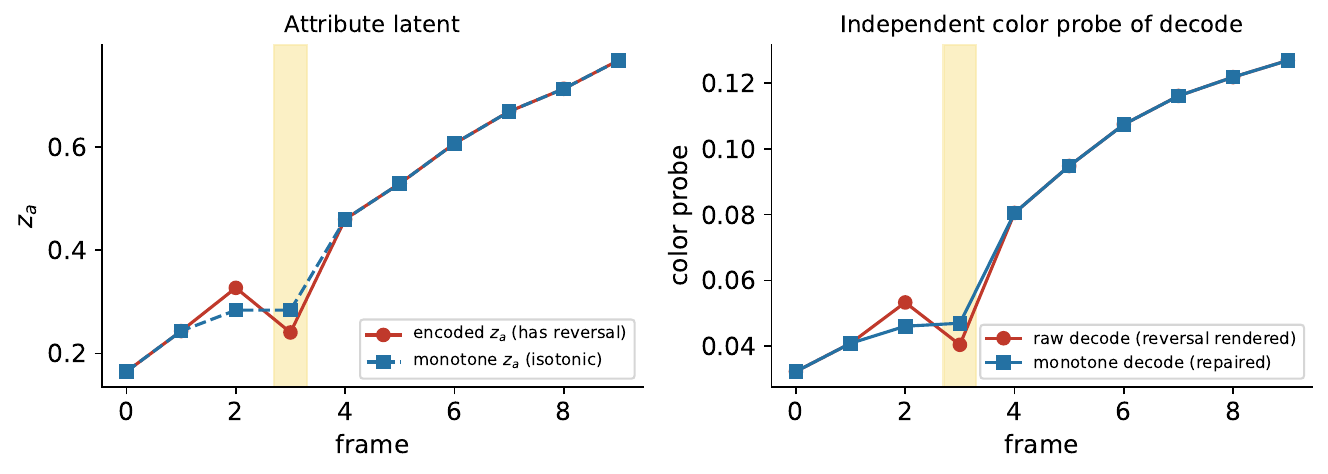}
\caption{\textbf{Structural repair of an injected reversal.} Left: the encoded attribute latent
$\za$ dips at the injected frame (shaded); isotonic projection removes the dip. Right: the
readout-independent color probe of the \emph{decode}---raw decoding renders the reversal, while
decoding the monotone $\za$ raises the probe at the injected frame, a genuine attribute increase.}
\label{fig:traj}
\end{figure}

\begin{figure}[t]
\centering
\includegraphics[width=0.86\linewidth]{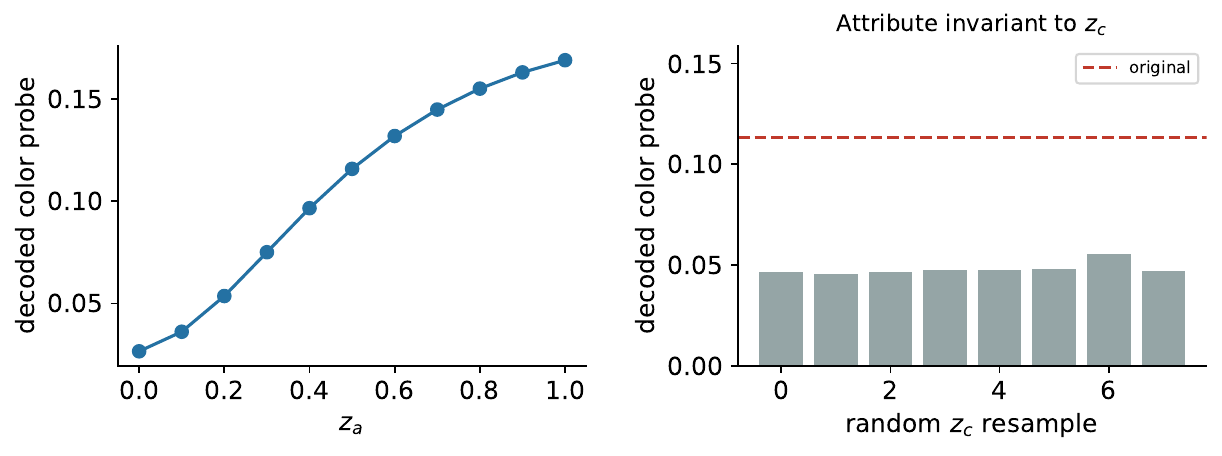}
\caption{\textbf{Disentanglement.} Left: the decoded attribute (color probe) responds monotonically
to $\za$ (a large sensitivity gap is what gives the monotone constraint teeth). Right: the attribute
is nearly invariant to random resampling of $\zc$, confirming nuisance and attribute are separated.}
\label{fig:disentangle}
\end{figure}

\begin{figure}[t]
\centering
\includegraphics[width=0.86\linewidth]{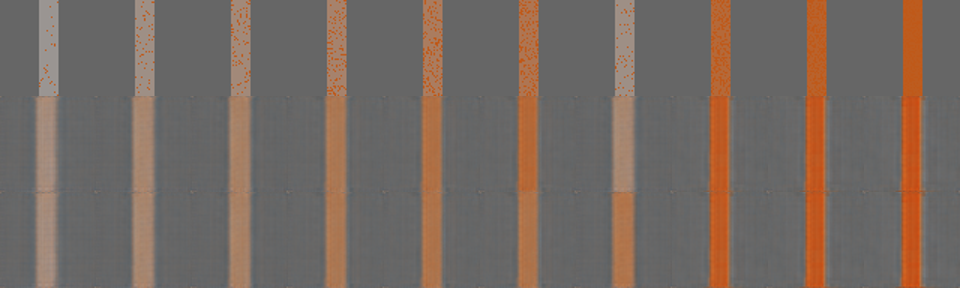}
\caption{\textbf{Genuine attribute repair in the controlled renderer.} Rows: input sequence with an
injected reversal (7th frame is clean); reconstruction from the raw encoded attribute (reproduces
the reversal); decode from the monotone-projected attribute latent (the 7th frame is rendered as a
genuinely rusted bar, same position/width/background). Unlike guidance, the readout-independent
color probe rises to the neighbouring level---the attribute is truly added.}
\label{fig:monotone}
\end{figure}

\subsection{The cost of post-hoc autoencoding, and why it is not the constraint}
\label{sec:cost}

Applying the constraint by encoding every frame and decoding it back is not free, and we measure the
cost rather than assume it away (24 SD sequences; Table~\ref{tab:cost}). The autoencoding round-trip
degrades fidelity (LPIPS $0.155$, PSNR $21.3$\,dB), \emph{halves} the frame-to-frame variation
($0.311\!\to\!0.159$, i.e.\ it over-smooths), loses $38\%$ of the net progress travelled from the
first frame ($0.208\!\to\!0.128$), and shifts the distribution (CLIP-FD $0.133$). These are exactly
the blur, flicker-suppression and motion-smoothing costs one should worry about.

Crucially, the \emph{constraint} is not what causes them: the monotone track is identical to the
plain reconstruction on every metric, because on already-monotone sequences the isotonic projection
does nothing (Proposition~5). The entire cost is the autoencoder. This is the
strongest argument for not doing the constraint post-hoc at all---and motivates the next section.

\begin{table}[t]
\caption{Cost of the post-hoc autoencoding pipeline (24 SD sequences). The constraint adds nothing
over plain reconstruction; the autoencoder is the whole cost.}
\label{tab:cost}
\begin{center}\small
\resizebox{\columnwidth}{!}{%
\begin{tabular}{lccccc}
\toprule
\textbf{Track} & \textbf{LPIPS $\downarrow$} & \textbf{PSNR $\uparrow$} & \textbf{Frame-to-frame var.} & \textbf{Net progress} & \textbf{CLIP-FD $\downarrow$} \\
\midrule
Original (SD)      & --      & --     & $0.311$ & $0.208$ & $0$ \\
AE reconstruction  & $0.155$ & $21.3$ & $0.159$ & $0.128$ & $0.133$ \\
Monotone (ours)    & $0.155$ & $21.3$ & $0.159$ & $0.128$ & $0.133$ \\
\bottomrule
\end{tabular}}
\end{center}
\end{table}

\subsection{Enforcing the partial order inside the generative dynamics}
\label{sec:dynamics}

The construction so far is applied post-hoc (encode frames, project $\za$, decode), which
Section~\ref{sec:cost} shows carries a real autoencoding cost. We now move it \emph{inside the
generative dynamics} and generalize the scalar attribute to a \emph{vector-valued partial order},
addressing the two most restrictive aspects of the formulation.

\paragraph{Vector attribute and partial order.} Many irreversible processes are not a single visual
scalar. We therefore use a $3$-dimensional attribute $\za=(\text{rust},\text{erosion},\text{char})$,
each component irreversible, so the constraint is the \emph{partial order} $\za^{(t)}\preceq\za^{(t+1)}$
(componentwise), not a total order on a scalar. In the controlled domain each component drives a
distinct visual channel (chroma, material loss via bar width, darkening) with its own CLIP-free probe.

\paragraph{Monotone dynamics (generation-time).} A latent dynamics generates the whole trajectory from
$(\za^{(0)},\zc)$:
\begin{equation}
\begin{split}
\text{ours:}\quad & \za^{(t+1)}=\za^{(t)}+\mathrm{softplus}\big(f_\theta(\za^{(t)},\zc,t)\big),\\
\text{baseline:}\quad & \za^{(t+1)}=\za^{(t)}+f_\theta(\za^{(t)},\zc,t).
\end{split}
\end{equation}
The softplus makes every component's increment non-negative, so the partial order holds \emph{by
construction at generation time}---no projection, no readout, nothing to game. Both dynamics are
trained on short monotone trajectories ($T{=}8$) and evaluated on $2\times$ longer rollouts ($T{=}16$)
with injected dynamics noise, i.e.\ out of the training regime.

\paragraph{Results (Fig.~\ref{fig:dynamics}).} The latent-space comparison is definitional, not
empirical---softplus increments are $\ge0$ so $V_{\text{latent}}=[0,0,0]$ by construction, while adding
noise to unconstrained increments gives $\approx\!0.5$ negative-increment fraction by symmetry; we state
it once for completeness and do not count it as a result. What is informative is
that under the \emph{independent} probes the ordering carries over: at noise $0$ ours attains $V=0.00$ on
rust versus $0.58$ for the baseline, and at the highest noise ours is $[0.18,0.35,0.11]$ against
$[0.45,0.54,0.46]$.
Residual probe violations for ours (e.g.\ rust up to $0.18$ at high noise) are largely \emph{measurement-side}, not constraint-side: the erosion probe
is a discrete pixel count and the probes are imperfectly faithful---precisely the defect $\kappa$ of
Proposition~3.

\paragraph{Fit cost.} Enforcing the order is not obviously free---on synthetic data the four monotone
constructions we tried (softplus, projected-clamp, an endpoint-informed variant, and a saturating gate;
App.~S3, Table~\ref{tab:dynbaselines}) all fit the trajectory a bit worse than the
unconstrained dynamics ($\approx\!2\times$ MSE). With four seeds this gap is not statistically
significant ($t\!\approx\!1.6$), and since the ground-truth path is monotone it may well be an
optimization effect rather than an intrinsic cost; we report it as an observation and do not build on it.

\begin{table}[t]
\caption{\textbf{All four monotone variants incur a similar ${\approx}2.2\times$ fit cost.} Trajectory
fit at zero dynamics noise (mean$\pm$std, 4 seeds): unconstrained dynamics vs.\ four monotone
constructions. Spreads overlap, so we read only that none---including the endpoint-informed variant and a
trivially-pausing gate---falls below the unconstrained fit; we claim no universal cost.}
\label{tab:dynbaselines}
\begin{center}\small
\resizebox{\columnwidth}{!}{%
\begin{tabular}{lccc}
\toprule
\textbf{Dynamics} & \textbf{traj.\ MSE} & \textbf{endpoint err} & \textbf{ratio} \\
\midrule
unconstrained & $0.033\pm0.009$ & $0.210\pm0.072$ & $1.00\times$ \\
softplus (ours) & $0.072\pm0.047$ & $0.220\pm0.051$ & $2.22\times$ \\
projected-clamp & $0.070\pm0.037$ & $0.228\pm0.076$ & $2.15\times$ \\
endpoint-conditioned$^{*}$ & $0.073\pm0.048$ & $0.226\pm0.057$ & $2.26\times$ \\
saturating-gated & $0.075\pm0.041$ & $0.257\pm0.080$ & $2.30\times$ \\
\bottomrule
\end{tabular}}
\end{center}
{\footnotesize $^{*}$given the true final attribute as input (an endpoint-informed variant).}
\end{table}

\begin{figure}[t]
\centering
\includegraphics[width=0.98\linewidth]{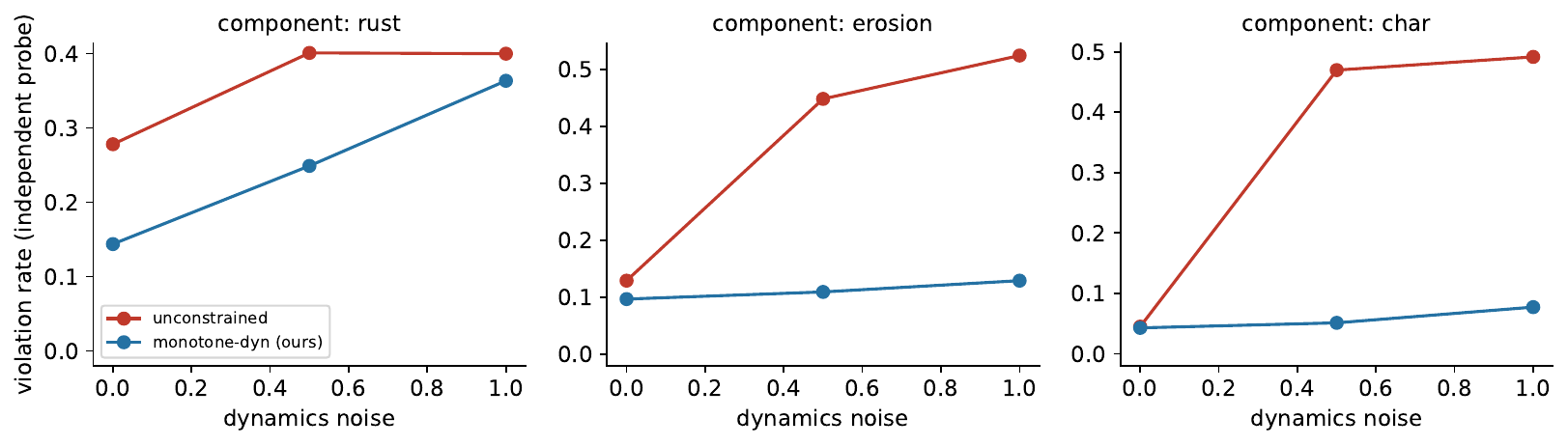}
\caption{\textbf{Monotone dynamics enforces a vector partial order at generation time.} Per-component
violation rate (independent probes) over $2\times$-horizon noisy rollouts. The unconstrained dynamics
reverses each irreversible component as noise grows; the monotone dynamics keeps latent violations at
exactly zero by construction and greatly reduces probe-measured violations. Residual probe violations
reflect probe unfaithfulness ($\kappa$ in Prop.~3), not constraint failure.}
\label{fig:dynamics}
\end{figure}

\subsection{Monotone is not enough: a structural progress floor}
\label{sec:floor}

Our own large-scale study (Sec.~\ref{app:chrono}) says the dominant real failure is \emph{stasis}, and
a constant trajectory satisfies monotonicity perfectly---so the constraint as stated is vacuous
against the failure that actually dominates. We therefore floor the increment,
$\za^{(t+1)}=\za^{(t)}+\epsilon+\mathrm{softplus}(f_\theta)$, which by
Proposition~4 guarantees \emph{both} no reversal and at least $(T{-}1)\epsilon$ total
progress, for any network, noise and horizon.

The floor conflates two requirements: an arrow-of-time constraint (the attribute must never decrease,
which a static clip trivially satisfies) and prompt-conditioned progress (it must change enough). The
monotone dynamics enforces only the former and is thus vacuous against stasis, while the fixed floor is a
crude surrogate for the latter, limited because a constant $\epsilon$ is prompt-blind and forces motion
through phases that should saturate; we therefore scope the floor as a stasis diagnostic rather than a
deployable objective, which would instead use a prompt-conditioned, phase-aware progress target.
Scaling the learned increment by $s{=}0.1$ to simulate an under-changing generator
(Table~\ref{tab:floor}), the monotone-only dynamics attains a perfect $V=0.000$ while $58\%$ of its
rollouts are static---so $V$ alone certifies nothing---whereas the floored dynamics keeps $V=0.000$ with
$0.753$ progress and $0\%$ static rollouts.

\begin{table}[t]
\caption{Monotonicity alone is vacuous against stasis. Under an under-changing generator ($s{=}0.1$)
the monotone-only dynamics scores a perfect $V$ while more than half its rollouts do nothing; the
structural floor gives both properties. ($64$ rollouts, $T{=}16$, dynamics noise $0.3$.)}
\label{tab:floor}
\begin{center}\small
\resizebox{\columnwidth}{!}{%
\begin{tabular}{lccc|ccc}
\toprule
& \multicolumn{3}{c|}{\textbf{normal generator} ($s{=}1$)} & \multicolumn{3}{c}{\textbf{under-changing} ($s{=}0.1$)} \\
\textbf{Dynamics} & $V\downarrow$ & progress $\uparrow$ & static\% $\downarrow$ & $V\downarrow$ & progress $\uparrow$ & static\% $\downarrow$ \\
\midrule
unconstrained            & $0.449$ & $0.665$ & $0\%$ & $0.422$ & $0.099$ & $5\%$ \\
monotone only            & $\mathbf{0.000}$ & $0.625$ & $0\%$ & $\mathbf{0.000}$ & $0.080$ & $\mathbf{58\%}$ \\
monotone $+$ floor (ours)& $\mathbf{0.000}$ & $\mathbf{0.783}$ & $\mathbf{0\%}$ & $\mathbf{0.000}$ & $\mathbf{0.753}$ & $\mathbf{0\%}$ \\
\bottomrule
\end{tabular}}
\end{center}
\end{table}

\section{Validation Beyond the Synthetic Renderer}
\label{sec:sd}

We use Stable Diffusion (SD-Turbo, \cite{sdturbo,ldm})
as a \emph{controllable renderer}: interpolating the text conditioning from a start-state to an
end-state prompt at level $\alpha$ over a bank of fixed initial latents yields a
(nuisance $\times$ attribute) image grid with known factors and hence free swap targets
$\mathrm{grid}[\ell_j,\alpha_i]$. We train the disentangled autoencoder of
Section~\ref{sec:method} on this grid (rusting nails) and run the identical injected-reversal
protocol on SD-manifold images. The mechanism transfers: violations are eliminated
($0.143\!\to\!0.000$ over 8 sequences), the repair is genuine (readout-independent probe recovers
$0.22\pm0.07$ of the gap, not a readout artifact), and faithfulness is present ($\za$ sensitivity
gap $0.16$; Fig.~\ref{fig:sd}). It also \emph{quantifies the bottleneck} the analysis predicts:
disentanglement is leakier than in the controlled domain ($\zc$-induced probe std $0.063$ vs.\
$0.012$) and autoencoder reconstruction of SD detail is imperfect, so the recovered attribute
fraction and fidelity are lower. This is exactly Assumption~1 degrading gracefully,
and it localizes the open work on real content to improving the disentangled autoencoder rather than
the constraint itself.

\begin{figure}[t]
\centering
\includegraphics[width=0.92\linewidth]{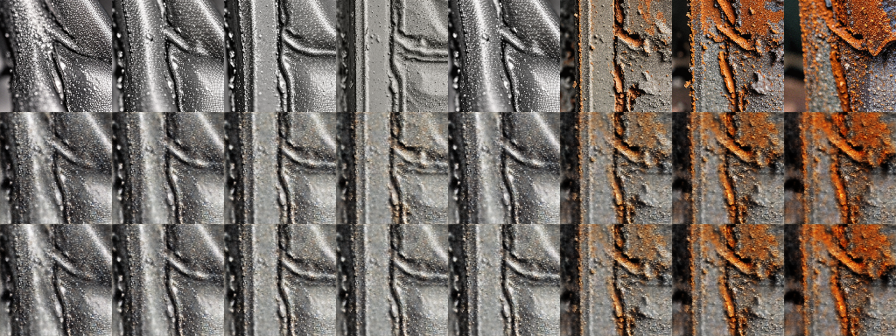}
\caption{\textbf{Mechanism transfers to Stable Diffusion content.} Rows: SD-generated sequence with
an injected reversal; disentangled-autoencoder reconstruction; decode from the monotone-projected
attribute latent. The clean$\to$rust progression is preserved and the reversal is removed; the
autoencoder softens SD detail (the fidelity bottleneck the analysis predicts on real content).}
\label{fig:sd}
\end{figure}

\section{Conclusion}

Whether generators respect the arrow of time is harder to measure than to state. Metrics of local
reversal are null-degenerate, and humans see no more backward motion in generated than in real video,
so what we can reliably measure is \emph{development}. Our main deliverable is a null-robust
progress/stasis protocol that survives null-testing where a per-clip violation rate and a generic
monotonicity score do not, and gives a consistent, human-validated diagnosis across seven T2V models:
they \emph{under-develop} irreversible attributes rather than reversing them. As a complementary
mechanism, post-hoc readout guidance is gameable, whereas enforcing a partial order \emph{by
construction} in a disentangled attribute latent removes the gameable readout, which we validate in
controlled and semi-synthetic settings where the faithful latent is available. We hope the protocol and
diagnosis serve as a reusable measurement instrument for irreversibility in video generation.

\bibliographystyle{IEEEtran}
\bibliography{refs}

\clearpage
\onecolumn
\begin{center}\Large\bfseries Supplementary Material\end{center}
\vspace{1em}
\twocolumn
\renewcommand{\thesection}{S\arabic{section}}
\setcounter{section}{0}

\section{Scope and Limitations}

We deliberately scope v1 to \emph{gradual, approximately scalar} irreversible attributes (rust,
rot, wilt, melt, aging, cumulative damage), for which a directional readout tracks progress well
(rank correlation $\approx 0.8$--$0.9$ in our measurement study). Two cases are out of scope and
left to future work: (i) \emph{abrupt threshold} processes such as burning, where the attribute is
near-constant and then jumps, so a scalar-monotone description fits poorly; and (ii) attributes whose
\emph{measurement} is not faithful---a simple chroma/darkness probe is unfaithful for e.g.\ apple
browning, which inflates apparent violations even when the latent constraint holds exactly (the
$\kappa$ of Prop.~\ref{prop:robust}); better probes, not a different constraint, are what is needed.
The \emph{multi-dimensional} case (melting couples shape, opacity and volume) is \emph{not} a
limitation of the formulation: Section~V-C instantiates a vector-valued partial order
and enforces it inside the generative dynamics.

\paragraph{Baselines and evidence scope.} On the dynamics side we now compare the softplus construction
against four alternatives (Table~IX): projected-clamp (a projected/constrained-diffusion
analogue), an endpoint-informed variant, and a saturating-gated variant---finding a similar $\approx\!2\times$
fit cost for all five in our synthetic setup (we do not claim this is universal). We do \emph{not} benchmark a continuous-time monotone neural
ODE/flow or a fully \emph{phase-aware} progress model with a learned stopping gate; these remain the right
comparisons for a deployable method and we make no superiority claim over them. On the real-model side
our attribute-\emph{specific} evidence now spans \emph{seven} models---CogVideoX-2b (five
processes, eight seeds) plus six ChronoMagic-Bench models over nine monotone-attribute categories
(Table~III)---and is uniform (near-zero progress, near-total stasis). We still distinguish this from the $1{,}050$-clip large-scale study, which uses the generic
monotonicity metric we argue is invalid and therefore supports only the \emph{stasis} finding, not a
claim of systematic \emph{reversal}. The honest generalized claim is thus: across seven models the
attribute is under-developed (near-total stasis, near-zero progress)---not that generation systematically
runs time backwards.

\section{Theoretical Guarantees (Elementary) and Their Limits}
\label{sec:theory}

The results below are \textbf{elementary}---each follows in a few lines from its assumptions---and we
include them not as technical contributions but because they turn informal claims into
\emph{measurable} decoder properties. Proofs are in Appendix~\ref{app:proofs}. Let the generator have
reachable set $\mathcal{R}=\{D(z):z\in\mathcal{Z}\}\subseteq\mathcal{X}$, let
$\alpha\colon\mathcal{X}\to\mathbb{R}$ be the (unobserved) true attribute, and $g$ a learned readout.

\begin{definition}[Irreversibility]
A clip $x_{1:T}$ \emph{respects} attribute $\alpha$ if $\alpha(x_t)\ge\alpha(x_{t-1})$ for all $t$;
its violation rate is $V(x_{1:T})=\frac{1}{T-1}\sum_{t=2}^{T}\mathbf 1[\alpha(x_t)<\alpha(x_{t-1})]$.
\end{definition}

\begin{proposition}[Frozen-readout guidance is not attribute-identifiable]
\label{prop:gaming}
The guidance program $x^\star\in\arg\min_{x\in\mathcal{R}}(g(x)-\tau)^2$ depends on $x$ only through
$g$. Hence if there is a ``false-high'' reachable point $x'\in\mathcal{R}$ with $g(x')=\tau$ but
$\alpha(x')<\tau$, then $x'$ is a global optimum whose true attribute is below target, and no
objective that is a function of $g$ alone can exclude it. Consequently guidance certifies the true
attribute only if $g$ is monotonically faithful to $\alpha$ on all of $\mathcal{R}$; learned readouts
are trained on the data manifold, but $\mathcal{R}$ includes off-manifold points produced by latent
optimization, where faithfulness is uncertified, so false-high points cannot be ruled out.
\end{proposition}
\emph{(Proof in App.~\ref{app:proofs}.)}

We now analyze the by-construction constraint, writing the decode $x=D(\za,\zc)$.

\begin{assumption}[Attribute-faithful, disentangled decoder]
\label{ass:faithful}
(A1, faithfulness) for every $\zc$, the map $\za\mapsto\alpha(D(\za,\zc))$ is non-decreasing; (A2,
disentanglement) $\alpha(D(\za,\zc))$ does not depend on $\zc$.
\end{assumption}

\begin{proposition}[Non-gameable monotonicity by construction]
\label{prop:construct}
Under Assumption~\ref{ass:faithful}(A1) \emph{and} either (i) a shared nuisance code across the
compared frames ($\zc^{(t)}=\zc^{(t-1)}$) or (ii) A2, any latent trajectory with
$\za^{(t)}\ge\za^{(t-1)}$ decodes to a clip with $\alpha(x_t)\ge\alpha(x_{t-1})$, i.e.\ $V=0$, for the
\emph{true} attribute. The guarantee involves no optimization against $g$ or $\alpha$, so it admits no
adversarial slack. Content preservation is a \emph{separate} claim that (A2) alone does not give:
(A2) says $\alpha$ is independent of $\zc$, but the non-attribute content is preserved only if the
decoder is \emph{factorized}, i.e.\ there is a nuisance map $\beta$ with $\beta(D(\za,\zc))$
independent of $\za$; we state this as an added condition rather than deriving it. \emph{Caveat:}
post-hoc decoding uses each frame's own
$\zc^{(t)}$, so condition (i) does not hold there and A2 (approximately, A2$'$) is required.
\end{proposition}
\emph{(Proof in App.~\ref{app:proofs}.)}

Assumption~\ref{ass:faithful} is not free---it is exactly what disentanglement training must deliver,
and it is \emph{empirically checkable}: (A1) is the monotone response of the probe to $\za$ (the
sensitivity gap, $0.12$ in Fig.~12, left) and (A2) is invariance of the probe to
$\zc$ (std $0.012$, right). This reduces correctness to a measurable property of the decoder, unlike
Proposition~\ref{prop:gaming}, whose failure is uncheckable at inference.

\paragraph{Real content admits only approximate disentanglement.} Exactness in
Proposition~\ref{prop:construct} rests on (A2) holding exactly; on natural content the attribute
leaks into $\zc$ (Section~VI), so we give a robust guarantee.

\begin{assumption}[Approximate faithful disentanglement]
\label{ass:approx}
There exist $\bar\alpha\colon\mathbb{R}\to\mathbb{R}$ and constants $m>0,\ \kappa\ge0$ with
(A1$'$) $\bar\alpha(u)-\bar\alpha(u')\ge m(u-u')$ for all $u\ge u'$ (faithful with margin $m$), and
(A2$'$) $\sup_{\za,\zc}\lvert\alpha(D(\za,\zc))-\bar\alpha(\za)\rvert\le\kappa$ (the
\emph{disentanglement defect}).
\end{assumption}

\begin{proposition}[Robust monotonicity under leakage]
\label{prop:robust}
Under Assumption~\ref{ass:approx}, enforcing $\za^{(t)}\ge\za^{(t-1)}$ gives
$\alpha(x_t)-\alpha(x_{t-1})\ge m\,(\za^{(t)}-\za^{(t-1)})-2\kappa\ge-2\kappa$; hence every residual
violation has magnitude at most $2\kappa$, and the true-attribute increase at a frame projected up by
$\Delta\za$ is at least $m\,\Delta\za-2\kappa$. Exactness ($\kappa=0$) recovers
Proposition~\ref{prop:construct}, and repair quality degrades gracefully and monotonically in the
\emph{measurable} defect $\kappa$---matching the larger $\kappa$ and lower recovery on
Stable-Diffusion content (Section~VI).
\end{proposition}

\paragraph{Monotonicity alone is vacuous against the dominant real failure.} Our large-scale study
(Sec.~III) found that real generators fail mostly by \emph{never developing} the
attribute, not by reversing it---and a constant trajectory satisfies monotonicity \emph{exactly}. A
constraint that only forbids decrease therefore awards a perfect score to a generator that does
nothing. We close this with a floored dynamics that remains structural.

\begin{proposition}[Two-sided structural guarantee: no reversal \emph{and} no stasis]
\label{prop:floor}
Let $\za^{(t+1)}=\za^{(t)}+\epsilon+\mathrm{softplus}\!\big(f_\theta(\za^{(t)},\zc,t)\big)$ with
$\epsilon>0$. Then for every $t$, every $f_\theta$, every noise realization and every horizon,
$\za^{(t+1)}-\za^{(t)}\succeq\epsilon\mathbf{1}$ componentwise; hence (i) the partial order holds
strictly ($V=0$), and (ii) the total progress is bounded below,
$\za^{(T)}-\za^{(0)}\succeq (T-1)\,\epsilon\mathbf{1}$. Under Assumption~\ref{ass:approx} the
\emph{true} attribute inherits both:
$\alpha(x_T)-\alpha(x_0)\ \ge\ m\,(T-1)\,\epsilon-2\kappa$.
The pure-monotone case $\epsilon=0$ yields only $\za^{(T)}\succeq\za^{(0)}$, which a constant
trajectory satisfies---so stasis is invisible to it.
\end{proposition}
\begin{proof}[Proof sketch]
$\mathrm{softplus}>0$ gives $\za^{(t+1)}-\za^{(t)}\ge\epsilon$ per component regardless of
$f_\theta$; telescoping over $t$ gives the progress bound; the attribute version follows by applying
(A1$'$)--(A2$'$) exactly as in Proposition~\ref{prop:robust}.
\end{proof}

The scalar statements apply componentwise to the $K$-dimensional attribute of
Section~V-C (with $m=\min_k m_k$, $\kappa=\max_k\kappa_k$); the partial order is just the
conjunction of $K$ scalar guarantees. The margin-dependence of Proposition~\ref{prop:robust} is borne out
empirically (Fig.~\ref{fig:ablation}): across $15$ autoencoders the readout-independent repair recovery is
$\approx0$ when the faithfulness margin is $\approx0$ and rises once it is positive ($r=0.90$). Since both
axes share the same proxy probe, we read this as corroboration, not an independent causal test.

\begin{figure}[t]
\centering
\includegraphics[width=0.6\linewidth]{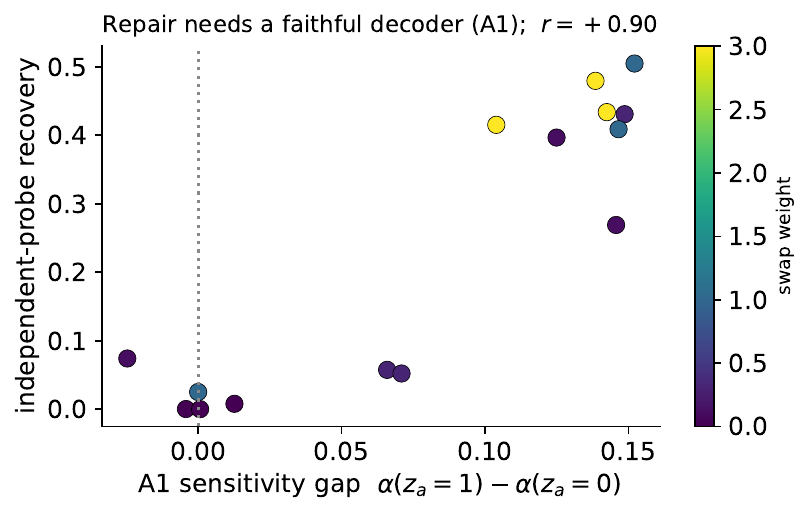}
\caption{\textbf{Repair reduces to decoder faithfulness (A1).} Each point is a trained autoencoder;
$x$-axis is the empirical faithfulness margin (probe response to $\za$), $y$-axis the genuine repair
recovery on injected reversals, color is the swap-disentanglement weight. Recovery is $\approx0$ when
the decoder is not faithful to $\za$ and jumps once the margin is positive ($r=0.90$), as
Proposition~\ref{prop:robust} predicts.}
\label{fig:ablation}
\end{figure}

\begin{proposition}[Minimal correction with a fixed-point property]
\label{prop:isotonic}
Let $a=(a_t)$ be the encoded attribute trajectory and $\hat a=\Pi_{\uparrow}(a)$ its isotonic
regression, the unique minimizer of $\lVert a-\hat a\rVert_2^2$ over non-decreasing sequences. Then
$\hat a=a$ iff $a$ is already non-decreasing (no change absent violations), and if $D$ is
$L$-Lipschitz in $\za$ the induced content change obeys
$\sum_t\lVert D(\hat a_t,\zc)-D(a_t,\zc)\rVert^2\le L^2\lVert a-\hat a\rVert_2^2$.
\end{proposition}
\emph{(Proof in App.~\ref{app:proofs}.)}

Finally we justify the training objective that is meant to deliver
Assumption~\ref{ass:faithful}(A2), namely the swap consistency loss.

\begin{proposition}[Swap consistency identifies the attribute/nuisance split]
\label{prop:swap}
Suppose data are produced by a factored renderer $x=R(u,v)$ with attribute $u$ and nuisance $v$ drawn
independently, $\alpha(R(u,v))=u$, and the swap objective is achieved:
$D\big(E_{\mathrm a}(R(u,v)),\,E_{\mathrm c}(R(u',v'))\big)=R(u,v')$ for all $(u,v),(u',v')$. Then the
decoded attribute $\alpha\big(D(E_{\mathrm a}(x),\zc)\big)$ is independent of $\zc$ (A2) and equals the
attribute encoded by $E_{\mathrm a}(x)$ (it recovers $u$). By contrast an adversarial independence
penalty enforces (A2) only in expectation and can collapse $\za$ (as we observed); swap consistency
enforces it pointwise.
\end{proposition}
\emph{(Proof in App.~\ref{app:proofs}.)}

\paragraph{Scope.} All guarantees are \emph{conditional} on (approximate) faithfulness: they convert the
unfalsifiable ``the fix is real'' into two measurable decoder quantities ($m,\kappa$), but $\alpha$ is
unobservable so $m,\kappa$ are estimated through proxy probes, and degrade on less controlled content.
Proposition~\ref{prop:gaming} is a
possibility result---no readout-only objective \emph{can} certify $\alpha$---not a proof that every
guidance variant fails in practice. The theory is an idealized account of the mechanism, not an
end-to-end guarantee.

\section{Experimental details}
\label{app:details}

\paragraph{Common setup.} All experiments run on a single NVIDIA A6000 (48\,GB) with PyTorch. Text-to-image
generation uses SD-Turbo \cite{sdturbo} (\texttt{stabilityai/sd-turbo}) at $512{\times}512$ with $3$
inference steps and classifier-free guidance scale $0$; graded attribute sequences are produced by
interpolating the CLIP text-encoder embedding from a start-state prompt $e_0$ to an end-state prompt
$e_1$ as $(1{-}\alpha)e_0+\alpha e_1$ while holding the initial latent fixed, so only the attribute
varies. Text-to-video uses CogVideoX-2b \cite{cogvideox} (fp16, model-CPU-offload, VAE tiling), $49$
frames, $50$ steps, guidance scale $6$. The attribute readout is a CLIP ViT-B/32
(\texttt{openai/clip-vit-base-patch32}) directional score $a(x)=\langle\phi(x),u\rangle$ with
$u=\mathrm{normalize}(\psi(\text{end})-\psi(\text{start}))$ ($\phi,\psi$ the image/text encoders). The
readout-\emph{independent} probe is a CLIP-free color statistic: mean $\max(R-B,0)$ (orange--brown
chroma) for rust/rot, or mean darkness $1-\overline{RGB}$ where noted. Identity is measured by LPIPS
(AlexNet). Isotonic projection is pool-adjacent-violators (PAVA).

\paragraph{Controlled ``rust-bar'' domain (Secs.~V,~\ref{sec:theory}).} A $64{\times}64$
renderer draws a vertical bar whose color interpolates grey $(0.6,0.6,0.6)\!\to\!$ orange
$(0.75,0.35,0.10)$ with attribute $a\in[0,1]$, plus orange rust speckles of density $0.4a$; nuisance
factors are bar $x$-position $\in[12,52]$, background shade $\in[0.2,0.7]$, and width $\in[8,24]$,
sampled independently. The autoencoder is a $4$-layer stride-2 conv encoder ($3\!\to\!32\!\to\!64\!\to\!64$)
to $(\za\in\mathbb R^1,\zc\in\mathbb R^8)$ and a mirrored deconv decoder; it is trained for $6000$ Adam
steps (lr $2{\times}10^{-3}$, batch $128$) on $\mathcal L=\lVert\hat x-x\rVert^2+\lVert D(\za,\zc^{(\pi)})-
R(a,\text{nuis}^{(\pi)})\rVert^2+0.5\lVert\za-a\rVert^2$, where $\pi$ is a random permutation (the swap term)
and $R$ the renderer. Evaluation injects a reversal at a mid-sequence frame over $8$ seeds/positions.

\paragraph{Latent dynamics and monotone-variant baselines (Sec.~V-C, Table~IX).}
The dynamics $f_\theta$ is a $3$-layer MLP $\mathrm{Linear}(3{+}z_c{+}1\!\to\!64)\!\to\!\mathrm{SiLU}\!\to\!
\mathrm{Linear}(64\!\to\!64)\!\to\!\mathrm{SiLU}\!\to\!\mathrm{Linear}(64\!\to\!3)$ (input: current
$\za^{(t)}\in\mathbb R^3$, nuisance $\zc\in\mathbb R^8$, and time $t/T$), trained for $3000$ Adam steps
(lr $2{\times}10^{-3}$, batch $128$) to predict the next-step attribute of short monotone trajectories
($T{=}8$) by MSE; the frozen faithful AE above provides $(\za,\zc)$. Noise is injected \emph{at generation
time} onto the raw increment, $d\!\leftarrow\!d+\sigma\,\varepsilon$, $\varepsilon\!\sim\!\mathcal N(0,I)$,
before the monotone map. The variants differ only in that map: \emph{unconstrained}
$\za^{(t+1)}{=}\za^{(t)}{+}d$; \emph{softplus} $+\mathrm{softplus}(d)$; \emph{projected-clamp}
$\max(\za^{(t)}{+}d,\za^{(t)})$ componentwise; \emph{saturating-gated} $+\sigma(g)\,\mathrm{softplus}(d)$
with a second gate head $g$ (network output widened to $6$); \emph{endpoint-informed} appends the true
$\za^{(T)}$ to the input. All variants share architecture, optimizer, step budget and data; Table~IX
reports trajectory fit at $\sigma{=}0$ over $4$ AE$\times$dynamics seeds, evaluated on $2\times$-longer
($T{=}16$) rollouts.

\paragraph{Stable-Diffusion semi-synthetic grid (Sec.~VI).} We render a
$(\text{nuisance}\times\text{attribute})$ grid of $L{=}48$ fixed latents $\times$ $A{=}8$ interpolation
levels for rusting nails, downscaled to $128{\times}128$; swap targets are grid look-ups
$\mathrm{grid}[\ell_j,\alpha_i]$. The autoencoder is a $4$-layer conv ($3\!\to\!32\!\to\!64\!\to\!128\!\to\!128$),
$\za\in\mathbb R^1$, $\zc\in\mathbb R^{16}$, trained $4000$ Adam steps (lr $1.5{\times}10^{-3}$, batch $64$)
with the same recon\,+\,swap\,+\,$\za$-supervision objective.

\paragraph{Gaming stress test and reject/resample (Sec.~IV).} A large reversal is injected
by replacing a late frame's latent with an early one. \emph{Guidance} optimises the frame's latent by
Adam on $(a-\hat a)^2+\lambda\lVert z-z_0\rVert^2$ (weak: $\lambda{=}0.05$, $40$ steps; strong:
$\lambda{=}0.002$, $150$ steps, lr $0.05$), backpropagating through the VAE and CLIP.
\emph{Reject/resample} scans the attribute conditioning forward on the shared latent and then over a
few fresh seeds, accepting the first sample whose readout is $\ge$ the previous frame's.

\paragraph{Ablation / theory link (Fig.~\ref{fig:ablation}).} In the rust-bar domain we train $15$
autoencoders spanning swap weight $\in\{0,0.1,0.3,1,3\}\times 3$ seeds ($6000$ steps each), and plot the
independent-probe repair recovery against the empirical faithfulness margin (probe response to $\za$),
reporting Pearson $r$.

\paragraph{Real-T2V violation measurement (Sec.~III).} Twelve irreversible prompts (ice, paper,
apple, nail, flower, candle, banana, bread, wood, copper, strawberry, leaf) generate CogVideoX-2b
clips; we subsample every third frame, compute the per-clip normalized readout, and report the
violation rate ($V=$ fraction of adjacent frames whose attribute decreases by more than $0.003$) and
the progress correlation.

\section{Dataset breadth: sixteen diverse irreversible processes}
\label{app:breadth}
To show the setting is not specific to a handful of hand-picked attributes, we synthesize graded
sequences for \textbf{sixteen} diverse irreversible processes spanning materials, biology, and
people---rusting, rotting, wilting, melting, burning, \emph{human-face aging}, banana ripening,
bread molding, copper corrosion, mud drying/cracking, wood weathering, log charring, ice-cream
melting, pumpkin decay, paper yellowing, and plant death (Fig.~1).
The directional readout tracks progress with \textbf{median $\rho=0.92$}, above the usable threshold
($\rho\ge0.7$) for \textbf{15 of 16} processes, with an $88\%$ reversal-detection rate
(Fig.~\ref{fig:breadth}). The single sub-threshold case is candle melting ($\rho=0.62$), whose
appearance is near-constant until a late collapse---the abrupt-transition regime we scope out. This
breadth confirms both that irreversibility is a pervasive, generically measurable structure and that
our gradual-attribute scope covers the large majority of common processes.

\begin{figure}[t]
\centering
\includegraphics[width=0.78\linewidth]{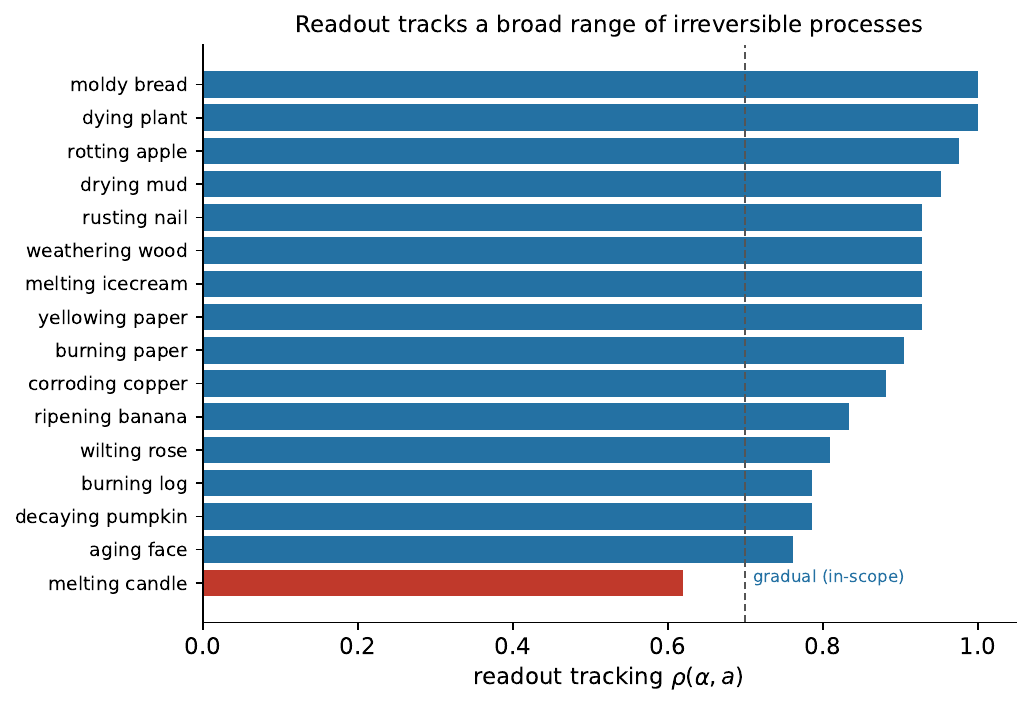}
\caption{\textbf{Readout tracking across sixteen processes.} Rank correlation $\rho(\alpha,a)$ per
process (blue: gradual, in-scope $\rho\ge0.7$; red: abrupt). Median $\rho=0.92$; $15/16$ in-scope.}
\label{fig:breadth}
\end{figure}

\section{Multiple violations: isotonic vs.\ cumulative clamp}
\label{app:multiviol}
On one reversal a cumulative clamp ($\za^{(t)}\!\leftarrow\!\max(\za^{(t)},\za^{(t-1)})$) already
attains $V=0$. Its weakness appears under \emph{multiple} violations, where it propagates an early
spike to all later frames, whereas isotonic regression is the global minimum-$L_2$ monotone fit and
pools only locally. Injecting $1/2/4$ reversals (rust-bar AE, 8 seeds), both reach $V=0$, but isotonic
has smaller total latent modification $\lVert\hat\za-\za\rVert_2$ ($0.14/0.16/0.18$ vs.\ clamp's
$0.18/0.24/0.28$) and changes non-violated frames much less as violations accumulate ($0.0034/0.0047/0.0138$
vs.\ clamp's $0.0000/0.0062/0.0247$). Isotonic's optimality is therefore what matters once more than a
single reversal is present.

\section{Reject/resample details}
On the injected reversal, reject/resample locates a genuinely rusted frame (color probe
$0.009\to0.23$, matching neighbours) but via a fresh sample that changes object identity (LPIPS
$0.72$ to the intended frame) and requires several full regenerations per violation, motivating the
by-construction approach.

\section{Detailed proofs}
\label{app:proofs}

Throughout, $\mathcal{R}=\{D(z):z\in\mathcal Z\}$ is the generator's reachable set, $\alpha$ the true
attribute, $g$ the learned readout, and $x=D(\za,\zc)$ a decode.

\paragraph{Proposition~\ref{prop:gaming} (guidance is not attribute-identifiable).}
The guidance program is $\min_{x\in\mathcal R} J(x)$ with $J(x)=(g(x)-\tau)^2$ (a content anchor only
shrinks $\mathcal R$ to a neighbourhood $\mathcal R_0\subseteq\mathcal R$ and does not change the
argument). Since $J$ depends on $x$ only through $g(x)$, it is constant on every level set
$\mathcal L_c=\{x\in\mathcal R_0:g(x)=c\}$. Let $c^\star$ minimise $|c-\tau|$ over attained values;
then $\arg\min J=\mathcal L_{c^\star}$, and \emph{every} point of $\mathcal L_{c^\star}$ is a global
optimum irrespective of its $\alpha$-value. In particular
$\min\{\alpha(x):x\in\arg\min J\}=\min\{\alpha(x):x\in\mathcal L_{c^\star}\}$, so a global optimum
$x^\star$ exists with
\begin{equation}
\alpha(x^\star)=\min_{x\in\mathcal R_0:\,g(x)=c^\star}\alpha(x).
\end{equation}
\emph{Corollary (gaming gap).} If $c^\star\ge\tau$ (the target readout is attainable) the shortfall
below any desired level $\alpha_0$ is
$\alpha_0-\alpha(x^\star)=\alpha_0-\min_{g(x)\ge\tau}\alpha(x)$, i.e.\ exactly the
\emph{readout infidelity} $\sup\{\alpha_0-\alpha(x): x\in\mathcal R_0,\,g(x)\ge\tau\}$: guidance is
faithful iff $g\ge\tau\Rightarrow\alpha\ge\alpha_0$ on $\mathcal R_0$. If $g=\phi\circ\alpha$ with
$\phi$ strictly increasing on $\mathcal R_0$ then $g(x)=\tau\Rightarrow\alpha(x)=\phi^{-1}(\tau)$ and
the gap is zero; a learned $g$ is trained on the data manifold $\mathcal D\subsetneq\mathcal R_0$ and
carries no such certificate off $\mathcal D$, where latent optimisation places $x^\star$. This is the
non-robustness of learned similarity functions; Section~IV realises it with $g$ a CLIP
similarity, $x^\star$ reaching $\tau$ while the independent probe of $\alpha$ is flat. \hfill$\square$

\paragraph{Proposition~\ref{prop:construct} (non-gameable monotonicity).}
Fix $t$ and let $\zc$ be the nuisance code carried across frames $t-1,t$ (identity is held during a
short window; otherwise invoke (A2)). By (A1), $u\mapsto\alpha(D(u,\zc))$ is non-decreasing, so
$\za^{(t)}\ge\za^{(t-1)}$ implies $\alpha(D(\za^{(t)},\zc))\ge\alpha(D(\za^{(t-1)},\zc))$. The left
side is $\alpha(x_t)$ and, since $x_{t-1}=D(\za^{(t-1)},\zc)$, the right side is $\alpha(x_{t-1})$;
hence $\alpha(x_t)\ge\alpha(x_{t-1})$ and $V=0$. No functional of $g$ or $\alpha$ is optimised, so
there is no level set on which to hide a false-high point. \emph{Content preservation is a separate
condition, not a corollary of} (A2): (A2) makes the \emph{attribute} $\alpha(D(\za,\zc))$ independent
of $\zc$, but this does not imply that changing $\za$ leaves the non-attribute content fixed---that
requires the decoder to factorize so that $\zc$ alone determines content (equivalently, a nuisance map
insensitive to $\za$), which we state as an \emph{additional} assumption, consistent with the main-text
caveat in Sec.~V. \hfill$\square$

\paragraph{Proposition~\ref{prop:robust} (robust monotonicity under leakage).}
Write $\alpha_t:=\alpha(x_t)=\alpha(D(\za^{(t)},\zc^{(t)}))$. By (A2$'$),
$\lvert\alpha_t-\bar\alpha(\za^{(t)})\rvert\le\kappa$ for every $t$. Hence
\begin{equation}
\begin{split}
\alpha_t-\alpha_{t-1}\ &\ge\ \big(\bar\alpha(\za^{(t)})-\kappa\big)-\big(\bar\alpha(\za^{(t-1)})+\kappa\big)\\
&=\bar\alpha(\za^{(t)})-\bar\alpha(\za^{(t-1)})-2\kappa\\
&\ge\ m\,(\za^{(t)}-\za^{(t-1)})-2\kappa,
\end{split}
\end{equation}
using (A1$'$) and $\za^{(t)}\ge\za^{(t-1)}$ in the last step. Thus $\alpha_t-\alpha_{t-1}\ge-2\kappa$
(residual violations bounded by $2\kappa$), and a frame projected up by $\Delta\za=\za^{(t)}-\za^{(t-1)}$
gains at least $m\,\Delta\za-2\kappa$ in true attribute. Setting $\kappa=0$ recovers
Proposition~\ref{prop:construct}; the bound is monotone in $\kappa$, which is directly measurable as
the probe's $\zc$-induced spread. \hfill$\square$

\paragraph{Proposition~\ref{prop:isotonic} (minimal correction with a fixed-point property).}
$\Pi_{\uparrow}$ is Euclidean projection onto the monotone cone $\mathcal C=\{\hat a:\hat a_1\le\dots\le\hat a_T\}$,
a nonempty closed convex set, so the minimiser is unique. Projection onto a convex set is the identity
on the set, giving $\hat a=a\iff a\in\mathcal C$ (no change without a violation). For the content bound,
$L$-Lipschitzness of $D(\cdot,\zc)$ gives
$\lVert D(\hat a_t,\zc)-D(a_t,\zc)\rVert^2\le L^2(\hat a_t-a_t)^2$; summing over $t$ yields
$\sum_t\lVert D(\hat a_t,\zc)-D(a_t,\zc)\rVert^2\le L^2\lVert a-\hat a\rVert_2^2$. Finally, because
$\hat a$ is the Euclidean projection onto the monotone cone it is \emph{by definition} the
minimum-norm monotone correction: $\lVert a-\hat a\rVert_2\le\lVert a-c\rVert_2$ for every monotone
$c$. Taking $c=\mathrm{cummax}(a)$ (the cumulative-max clamp, which is monotone) gives in particular
$\lVert a-\hat a\rVert_2\le\lVert a-\mathrm{cummax}(a)\rVert_2$, so isotonic regression never modifies
the sequence more than the cumulative clamp does---matching the empirical ordering in
App.~\ref{app:multiviol}. \hfill$\square$

\paragraph{Proposition~\ref{prop:swap} (swap $\Rightarrow$ disentanglement).}
By hypothesis $\alpha(R(u,v))=u$ and the swap identity holds. For any $x=R(u,v)$ and any nuisance code
$\zc=E_{\mathrm c}(R(u',v'))$,
\begin{equation}
\begin{split}
\alpha\big(D(E_{\mathrm a}(x),\zc)\big)&=\alpha\big(D(E_{\mathrm a}(R(u,v)),E_{\mathrm c}(R(u',v')))\big)\\
&=\alpha\big(R(u,v')\big)=u,
\end{split}
\end{equation}
which does not depend on $(u',v')$, hence not on $\zc$: this is (A2). The common value $u$ is the
attribute of the image encoded into $\za$, so $\za$ determines the attribute and $\zc$ carries none of
it. (An expectation-level independence penalty admits the degenerate solution $\za\equiv\mathrm{const}$,
which the pointwise swap identity excludes because it must reconstruct $R(u,v')$ for every $u$.)
\hfill$\square$

\end{document}